\documentclass{article} 
\PassOptionsToPackage{table}{xcolor}
\usepackage{iclr2027_conference,times}
\iclrfinalcopy

\usepackage{amsmath,amsfonts,bm}

\def\eqref#1{equation~\ref{#1}}

\def\1{\bm{1}}

\DeclareMathAlphabet{\mathsfit}{\encodingdefault}{\sfdefault}{m}{sl}
\SetMathAlphabet{\mathsfit}{bold}{\encodingdefault}{\sfdefault}{bx}{n}

\usepackage{hyperref}
\usepackage{amsmath,amssymb,amsthm,booktabs,graphicx,url,multirow,enumitem}
\usepackage{algorithm,algorithmic}
\newtheorem{proposition}{Proposition}

\theoremstyle{remark}

\usepackage{pifont}
\usepackage{iftex}
\ifXeTeX
  \usepackage{fontspec}
\fi

\title{Volatility-Clustering Adaptation for Financial Time Series}

\author{Manh Nguyen, Minh Hoang Nguyen, Huu Hiep Nguyen, Van Dai Do, Hung Le\thanks{Corresponding Author}\\
Applied Artificial Intelligence Initiative, Deakin University, Australia\\
\texttt{\{manh.nguyen, s223669184, huu.n, v.do, thai.le\}@deakin.edu.au}
}

\begin{document}
\maketitle
\begin{abstract}
Time-series foundation models are increasingly adapted to new domains through fine-tuning on target data, under the implicit assumption that more target data yields better forecasts. We show that this assumption can fail in financial forecasting, where individual price changes are difficult to predict, but large moves tend to cluster, creating alternating calm and turbulent periods. Using financial foundation models trained on price bars of open, high, low, close, and volume, we argue that adapting to financial domains requires training signals beyond next-token prediction.
We introduce \textbf{Volatility-Clustering Adaptation (VCA)}, which augments next-token cross-entropy with a differentiable penalty on the autocorrelation of squared returns, the standard statistical signature of volatility clustering. This additional objective provides a multi-step training signal by matching the resulting dependence structure of autoregressive rollouts to those of the realized future.
Across three asset sets and two evaluation conventions, VCA improves adaptation over the pre-trained model, with the strongest gains under the primary evaluation (\textsc{fore}), driven primarily by reduced variance error.
Overall, our results suggest that effective financial adaptation requires objectives that capture domain-specific temporal structure beyond token-level prediction.
Code is publicly available at \url{https://github.com/DA2I2-SLM/VCA}.
\end{abstract}

\section{Introduction}
\label{sec:intro}

Time-series foundation models are pre-trained on large and diverse corpora and then applied to previously unseen time series \citep{lagllama2023, timesfm2024,moirai2024,moment2024,chronos2024}. Financial foundation models such as Kronos
\citep{kronos2025} extend this paradigm to finance by tokenizing open, high, low, close, and volume (OHLCV) bars and autoregressively
generating future price paths.
Adapting such models to a target market typically involves fine-tuning on historical data, with the implicit assumption that using more target data should improve performance. That assumption can fail because fine-tuning can distort useful pre-trained features and reduce out-of-distribution robustness \citep{kumar2022finetune}, or overwrite capabilities acquired during pre-training \citep{kirkpatrick2017ewc}. The problem is particularly relevant in finance, where returns, defined as relative changes in asset prices, are weakly predictable but volatility is strongly autocorrelated and varies across regimes \citep{engle1982arch,bollerslev1986garch,cont2001stylized}. Long historical windows can therefore mix very different volatility regimes, so adding more data may introduce patterns that are less relevant to the regime at deployment.

This raises two separate questions for adaptation: \emph{what structure the training objective should enforce}, and, secondarily, \emph{which data to fine-tune on}. For the training objective, we encode volatility clustering, a classical stylized fact of financial returns, as an auxiliary signal, since a next-token prediction loss gives no explicit signal about the shape of a multi-step rollout used for test-time evaluation. In particular, we introduce \textbf{Volatility-Clustering Adaptation (VCA)}, which matches the \emph{autocorrelation function of squared returns (ACF$^2$)} between autoregressive rollouts and the corresponding realized future paths (Figure~\ref{fig:overview}). This directly trains the model to reproduce volatility clustering over the evaluation horizon.
To make this objective differentiable, we use a Gumbel-softmax straight-through rollout \citep{jang2017gumbel, bengio2013straightthrough}.
Additionally, we compute ACF statistics at the batch level rather than from individual trajectories, providing a more stable training signal. This formulation brings stylized-fact-driven objectives, analogous to physics-informed learning ~\citep{raissi2019pinn, wiese2020quantgans}, into the adaptation of an existing financial foundation model rather than training one from scratch.
Regarding data selection, we ask a retrospective question: does fine-tuning on a chronological slice whose realized volatility matches the test period outperform fine-tuning on the full historical set? We hypothesize that long historical windows mix volatility regimes that may differ from the one at deployment, so a smaller matched slice could supply more relevant signal than the full set. This setup examines whether regime alignment drives adaptation gains beyond data volume. To verify, we use the realized volatility of the test period for matching, making this a diagnostic study of volatility-regime alignment.

\begin{figure}[t]
\centering
\includegraphics[width=\textwidth]{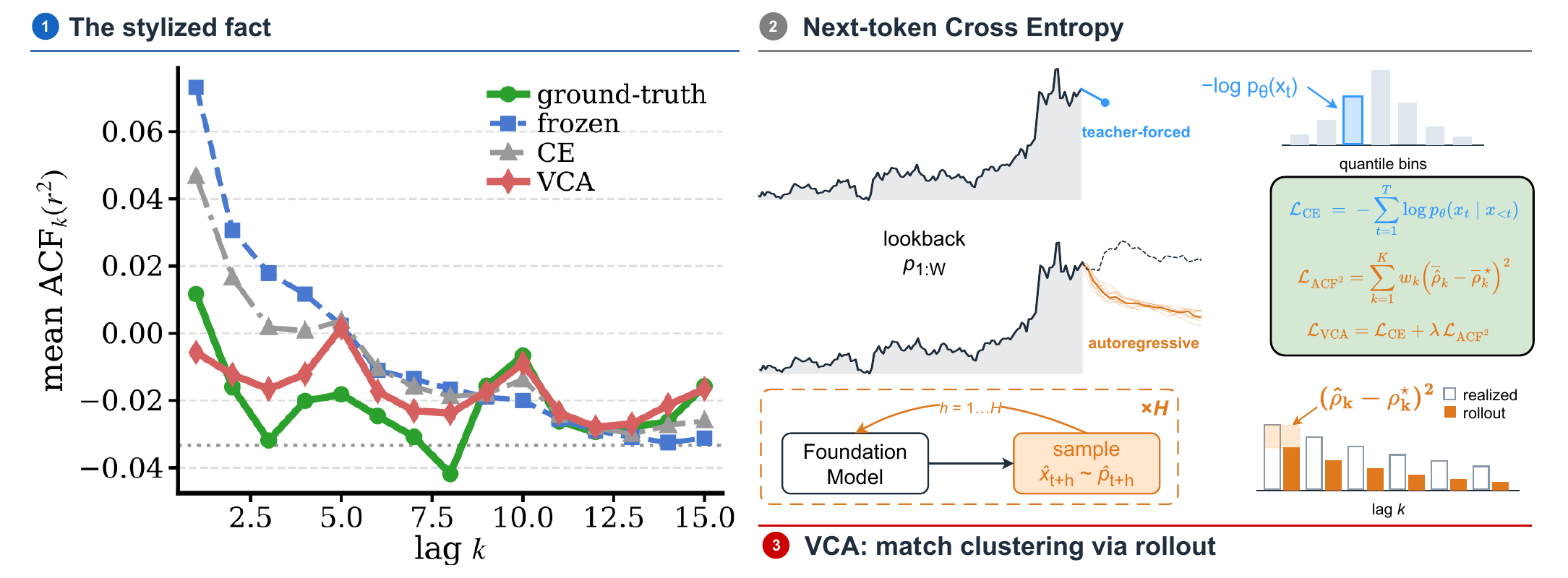}
\caption{Overview of Volatility-Clustering Adaptation (VCA). \textbf{\textcolor{blue}{\ding{202}}} Mean ACF of squared returns against lag $k$ on CSI300: the realized (ground-truth) path stays positive and well above the dotted zero-clustering null, the signature of volatility clustering. The frozen model and CE fine-tuning deviate from the ground-truth curve, while VCA matches its shape more closely. \textbf{\textcolor{gray}{\ding{203}} } The baseline objective $\mathcal{L}_{\text{CE}}$ is single-step and teacher-forced: it scores the next token against the ground truth but gives no signal about the shape of the multi-step path the model generates at test time. \textbf{\textcolor{red}{\ding{204}} } VCA rolls out autoregressively over the horizon and adds $\mathcal{L}_{\text{ACF}^2}$, which matches the \emph{batch-mean} ACF of squared returns $\bar{\rho}_k$ to the realized $\bar{\rho}_k^{\star}$. VCA's total objective is $\mathcal{L}_{\text{VCA}}=\mathcal{L}_{\text{CE}}+\lambda\mathcal{L}_{\text{ACF}^2}$.}
\label{fig:overview}
\end{figure}

Across major financial benchmarks such as CSI300, CSI500, and Crypto, VCA improves adaptation under both evaluation conventions, with gains driven primarily by lower variance error. In contrast, cross-entropy fine-tuning can be harmful, particularly on Crypto, while full-data VCA causes greater weight drift from the frozen backbone, which we use as a proxy for potential forgetting. Regime-matched historical data improves data efficiency at longer forecasting horizons, matching or outperforming full-data fine-tuning with less weight drift, but this benefit does not extend to shorter horizons. We observe similar results when adapting Chronos-T5-small, suggesting that these challenges are not specific to Kronos, highlighting the value of domain-specific path-level objectives and horizon-aware data selection for financial foundation-model adaptation.

Our contributions are:
\begin{itemize}[leftmargin=*,itemsep=1pt,topsep=2pt]

\item We introduce \textbf{Volatility-Clustering Adaptation (VCA)}, a differentiable objective that matches volatility clustering in autoregressive rollouts, with a controlled ablation isolating the contribution of ACF$^2$ from that of rollout-based training itself.

\item We show that volatility-matched data can match or outperform full-data fine-tuning at the longer horizon with substantially less weight drift, while broader data coverage remains more important at the shorter horizon.

\item We identify a volatility-trend direction that the ACF$^2$ loss cannot constrain and confirm it in both simulation and empirical results, providing a partial explanation for VCA's instability.

\end{itemize}

\section{Related Work}
\label{sec:related}

\paragraph{Time-series foundation models.}
Lag-Llama \citep{lagllama2023}, TimesFM \citep{timesfm2024}, Moirai \citep{moirai2024}, MOMENT \citep{moment2024}, and Chronos \citep{chronos2024} are pre-trained on large heterogeneous corpora and applied zero-shot across domains, using the same weights without modeling domain-specific structure. Kronos \citep{kronos2025} instead targets finance, tokenizing OHLCV bars into hierarchical discrete tokens and pre-training a decoder-only transformer on billions of K-lines. All of these works focus on optimizing forecasting accuracy, but none constrain adaptation to preserve domain-specific properties such as volatility clustering.

\paragraph{Adaptation, domain knowledge, and data selection for time series.}
Several approaches have been proposed to adapt time-series foundation models to target domains. Parameter-efficient methods use LoRA~\citep{hu2022lora, gupta2024lowrank} or specialized adaptation schemes such as TRACE~\citep{li2025trace} to update only a small fraction of model parameters. In-context fine-tuning adapts to a target distribution through related time-series examples without updating model weights at inference time~\citep{faw2025incontext}, while online adaptation updates forecasts as new observations arrive~\citep{lee2025lightweight}. Other work moves beyond conventional forecast loss, for example by fine-tuning toward downstream decision objectives~\citep{beichter2025decision}. In contrast, our adaptation changes the training objective itself by encoding a financial stylized fact, volatility clustering, through a differentiable path-level loss. 
Data selection provides a complementary route, with evidence that carefully selected subsets can outperform full-data fine-tuning, while simply matching the training data to the target distribution does not always improve transfer~\citep{xia2024less,kang2024data}.
In finance, DoubleAdapt~\citep{doubleadapt2023} adapts a stock-forecasting model to shifting regimes through a learned adapter. We instead use volatility regime to select historical data and find that such alignment can improve data efficiency at longer forecasting horizons.

\paragraph{Stylized facts in financial modeling.}
Autocorrelation in squared returns, together with weak autocorrelation in raw returns, is a classical financial stylized fact \citep{engle1982arch, bollerslev1986garch, cont2001stylized}. Their strength depends on sampling frequency \citep{drostnijman1993}, motivating our evaluation across forecast horizons. Quant GANs reproduce financial stylized facts directly \citep{wiese2020quantgans}, while physics-informed learning encodes known structure in differentiable objectives \citep{raissi2019pinn}. Exposure bias highlights the gap between teacher-forced training and autoregressive evaluation \citep{bengio2015scheduled, lamb2016professor}, and soft-DTW and DILATE address related gaps with differentiable path-shape losses \citep{cuturi2017softdtw, leguen2019dilate}.  These methods target generic path geometry rather than any domain-specific structure.
To our knowledge, we are the first to use a financial stylized fact as a differentiable path-level objective for adapting a pretrained foundation model.

\section{Method}
\label{sec:method}
We propose \emph{Volatility-Clustering Adaptation (VCA)}, which augments next-token prediction with a multi-step signal for volatility dependence. It matches the squared-return ACF of differentiable rollouts to that of the realized future, while preserving the original CE objective. 
We build on Kronos \citep{kronos2025}, a decoder-only foundation model for financial time series that tokenizes each open, high, low, close, and volume (OHLCV) bar into hierarchical discrete tokens and predicts the next bar autoregressively. Given a lookback window of $W$ bars, Kronos rolls out an $H$-bar forecast by sampling tokens one
step at a time and decoding them back to prices.

\subsection{Preliminaries}
\label{sec:overview}

\begin{figure}[t]
\centering
\includegraphics[width=\textwidth]{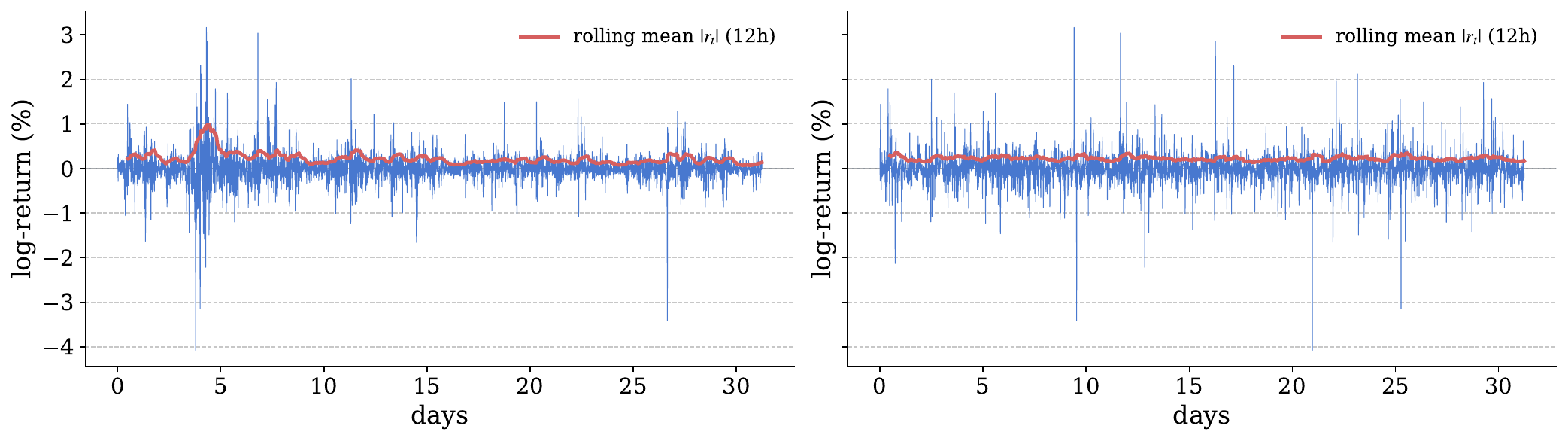}
\caption{Meaning of $\hat\rho_k>0$ on a real trajectory (BTCUSDT, 15-minute bars). 
Blue is the log-return $r_t$; red is the rolling mean $|r_t|$ (12h), a realized-volatility proxy. \textbf{(a)} Real order: volatility moves in sustained stretches, so a large $|r_t|$ predicts a large $|r_{t+k}|$. \textbf{(b)} Same returns, time-shuffled: the returns have the same values and distribution, but once their order is shuffled, the clustering disappears, and the envelope flattens. The ordering impact is 
the structure that $\hat\rho_k$ measures and what Eq.~\ref{eq:acfloss} trains the model to match.}
\label{fig:acf}
\end{figure}

\paragraph{Problem Formulation.}
A training sequence consists of the tokenized $W{+}H$ bars, $x_{1:W+H}$, spanning the lookback window and the forecast horizon. The model $p_\theta$ is trained to predict each token autoregressively, giving the next-token cross-entropy (CE) objective
\begin{equation}
\label{eq:ce}
\mathcal{L}_{\mathrm{CE}} \;=\; -\sum_{t=1}^{W+H}\log p_\theta(x_t\mid x_{<t}).
\end{equation}
This is the standard training objective for autoregressive time-series foundation models, and for financial foundation models such as Kronos in particular.
The CE objective is teacher-forced and one step ahead; that is, it sharpens next-token accuracy, but gives no signal about the shape of the $H$-step rollout that is actually produced and evaluated at test time. We next describe one such shape property, \textit{volatility clustering}, which Eq.~\ref{eq:ce} fails to model. Our method (\S\ref{sec:mainmethod}) explicitly
optimizes for this property.

\paragraph{The volatility-clustering stylized fact.}
Financial returns exhibit \emph{volatility clustering}. The autocorrelation function (ACF) of
\emph{squared} returns decays slowly and stays positive, while raw returns are near white noise
\citep{cont2001stylized}. For a price path $p_{1:H}$, we form log-returns $r_t=\log(p_t/p_{t-1})$ and let $\overline{r^2}$ denote their mean squared return. The sample ACF of squared returns at lag $k$ is
\begin{equation}
\label{eq:rho}
\hat\rho_k \;=\; \frac{\sum_{t}(r_t^2-\overline{r^2})(r_{t-k}^2-\overline{r^2})}
{\sum_{t}(r_t^2-\overline{r^2})^2}, \qquad k=1,\dots,K .
\end{equation}
$K$ introduces a trade-off between capturing longer-range dependence and maintaining reliable estimates. Larger $K$ captures slower-decaying clustering, but longer-lag estimates of $\hat\rho_k$ are noisier and sit closer to the no-clustering null.
Concretely, $\hat\rho_k>0$ means a large move tends to be followed by another large move $k$ steps later, regardless of direction. 
Figure~\ref{fig:acf} illustrates this temporal dependence directly: the same returns exhibit sustained volatility clustering in their original order, but not after time-shuffling.

\subsection{Volatility-Clustering Adaptation}
\label{sec:mainmethod}

Standard fine-tuning adapts foundation models using the objective in Eq.~\ref{eq:ce}. We augment it with a signal for the volatility dependence of the multi-step rollout. 

\paragraph{A differentiable ACF$^2$ loss.}
Computing the ACF loss on generated trajectories requires gradients to propagate through the autoregressive sampling process. Because Kronos predicts discrete tokens, naive sampling would break this gradient path. We therefore replace the categorical draw with a Gumbel-softmax straight-through estimator~\citep{jang2017gumbel, bengio2013straightthrough}. Specifically, letting $\ell_t$ denote the model logits over the token vocabulary, the forward pass feeds the $\arg\max$ token, matching test-time discretization, while the backward pass uses the soft relaxation
\begin{equation}
\label{eq:gumbelst}
q_t \;=\; \mathrm{softmax}\big((\ell_t+g_t)/\tau\big), \qquad g_t\sim\mathrm{Gumbel}(0,1),
\end{equation}
with Gumbel-softmax temperature $\tau$ annealed over training. This yields a fully differentiable autoregressive rollout $\hat p_{1:H}$ (one sample per window at train time, versus an $N$-rollout average at test time), from which we compute the per-window $\hat\rho_k$ of Eq.~\ref{eq:rho} and penalize its discrepancy to the realized future squared-return ACF $\rho^\star_k$ of the ground-truth path.

A per-window loss $(\hat\rho_k-\rho^\star_k)^2$ is a high-variance target even for a perfectly calibrated model, because $\hat\rho_k$ is a sample ACF computed from as few as $H{-}k$ squared returns. We therefore average \emph{before} squaring, over a batch of $B$ windows indexed by $b$,
\begin{equation}
\label{eq:batchacf}
\bar\rho_k \;=\; \frac{1}{B}\sum_{b=1}^{B}\hat\rho_k^{(b)}, \qquad
\bar\rho^\star_k \;=\; \frac{1}{B}\sum_{b=1}^{B}\rho^{\star(b)}_k,
\end{equation}
which reduces the sampling variance of each term before the squared difference is taken. The \emph{aggregate} training loss is then
\begin{equation}
\label{eq:acfloss}
\mathcal{L}_{\mathrm{ACF}^2} \;=\; \sum_{k=1}^{K} w_k\,
\big(\bar\rho_k-\bar\rho^\star_k\big)^2,
\qquad w_k=e^{-k/\gamma}.
\end{equation}
We retain the general $K$-lag form as the natural statement of the stylized fact where near lags carry most of the stylized-fact signal and empirical persistence decays with lag \citep{cont2001stylized} (Figure~\ref{fig:overview}, panel 1), so $w_k$ is a decay kernel rather than a normalized distribution over lags. 
When $K{=}1$ (our default), only $w_1=e^{-1/\gamma}$ enters the loss, so $\lambda$ and $\gamma$ affect the loss only through the product $\lambda e^{-1/\gamma}$, while normalizing $w_k$ would simply rescale $\lambda$.
The total objective is
\begin{equation}
    \label{eq:total-loss}
    \mathcal{L}=\mathcal{L}_{\mathrm{CE}}+\lambda\,\mathcal{L}_{\mathrm{ACF}^2},
\end{equation}
with plain CE recovered at $\lambda=0$. The ACF$^2$ term is thus the only difference from the CE baseline, so any gap between the two is attributable to this signal. However, this signal constrains only the \emph{dependence structure} of squared returns, not their overall \emph{level}. In particular, it cannot distinguish opposite trends in volatility level. We formalize this property below.

\begin{proposition}[The volatility-trend channel]
\label{prop}
Let $r_t=\sigma_0e^{gt}\varepsilon_t$, where $g$ controls the exponential trend in volatility, for $t=1,\dots,n$ with $\varepsilon_t$ i.i.d.\ and no clustering. Then (a) the law of the sample ACF $\hat\rho_k(r^2)$ is invariant under $g\mapsto-g$,
while (b) $\mathbb{E}\big[\sum_t r_t^2\big]=\sigma_0^2\sum_{t=1}^{n}e^{2gt}$ is strictly increasing in
$g$, with
$\mathbb{E}\big[\sum_t r_t^2\,;+g\big]\,/\,\mathbb{E}\big[\sum_t r_t^2\,;-g\big]=e^{2g(n+1)}$.
\end{proposition}
\begin{proof}
See Appendix~\ref{app:theory}.
\end{proof}

A volatility trend can substantially change the overall level of squared returns without changing the distribution of their squared-return ACF. Matching the ACF therefore does not control the overall volatility level, leaving the objective insensitive to the direction of a volatility trend. This provides a mechanistic explanation for a failure mode observed in our VCA runs, where some runs end with sharply inflated forecast variance despite a small ACF$^2$ validation gap (Section~\ref{sec:theory}).
In practice, this effect is minor on our primary backbone, Kronos: VCA improves on the frozen model across all three datasets and horizons (Table~\ref{tab:main}), with variance inflation remaining modest relative to the magnitude possible in Proposition~\ref{prop}. The effect is more visible on Chronos-T5-small (Appendix Table~\ref{tab:chronos_main}), a smaller and more general backbone. We discuss this limitation further in Appendix~\ref{sec:limitations}.

\paragraph{Why target the ACF gap, and not volatility level directly?}
Realized-variance level is a natural alternative target for a volatility-aware objective, but we find out empirically that its gap to a naive baseline is much less consistent across datasets than the ACF$^2$ gap. In contrast, ACF$^2$ shows a stable $\approx2\times$ gap and is the only consistent performer across all three datasets, motivating our choice of ACF$^2$ as the training signal (Appendix Table~\ref{tab:headroom}).
Intuitively, variance level is \textit{scale-dependent} and can be matched reasonably well by a naive baseline that tracks the unconditional scale, whereas ACF$^2$ captures \textit{scale-invariant} temporal dependence that such baselines cannot reproduce. This makes its headroom more stable across datasets and regimes.

\section{Experiments}
\label{sec:experiment}

We evaluate VCA across datasets and forecasting horizons against alternative adaptation objectives and classical volatility models, using both finance-specific and general-purpose time-series foundation models. Our ablations then study three aspects of the method: how regime matching affects fine-tuning data efficiency, how adaptation changes model weights, and how the main design choices ($K$, $W$, and $\lambda$) affect performance.

\subsection{Setup}\label{sec:setup}
\paragraph{Datasets.} We evaluate on three asset sets in the per-symbol OHLCV format consumed by the Kronos inference pipeline: \textbf{CSI300} and \textbf{CSI500}, standard equity benchmarks in prior return-forecasting work \citep{master2024,hist2021,doubleadapt2023} sourced from Qlib \citep{qlib2020}, and \textbf{Crypto}, 20 USDT pairs from the Binance API~\citep{binance_api}. More details are in Appendix Table~\ref{tab:data}.

\begin{table}[ht]
\centering
\caption{Main comparison under \textsc{fore}. The Frozen row (for reference) is raw RankIC and MAE; GARCH(1,1) and HAR-RV forecast only variance, so RankIC and Score are not applicable to them (\textit{-}). Every other row is a relative delta against Frozen in percentage (higher is better), and the last group averages both deltas and the Score over horizons.
Bold marks the best fine-tuned method in each column. \textbf{WR} denotes win rate, computed as the fraction of the eight per-horizon columns in which a method has the best result.
Per-horizon values are in Appendix Tables~\ref{tab:raw8}--\ref{tab:raw48}.
}
\label{tab:main}
\setlength{\tabcolsep}{2pt}
\resizebox{\textwidth}{!}{%
\begin{tabular}{clrrrrrrrr|rrrc}
\toprule
 &  & \multicolumn{2}{c}{$H{=}8$} & \multicolumn{2}{c}{$H{=}16$} & \multicolumn{2}{c}{$H{=}32$} & \multicolumn{2}{c}{$H{=}48$} & \multicolumn{3}{c}{average over $H$} & \\
\cmidrule(lr){3-4} \cmidrule(lr){5-6} \cmidrule(lr){7-8} \cmidrule(lr){9-10} \cmidrule(lr){11-13}
DS & Method & $\Delta_\text{RankIC}$ & $\Delta_\text{MAE}$ & $\Delta_\text{RankIC}$ & $\Delta_\text{MAE}$ & $\Delta_\text{RankIC}$ & $\Delta_\text{MAE}$ & $\Delta_\text{RankIC}$ & $\Delta_\text{MAE}$ & $\Delta_\text{RankIC}$ & $\Delta_\text{MAE}$ & Score & WR \\
\midrule
\multirow{6}{*}{\rotatebox[origin=c]{90}{CSI300}} & Frozen & +0.291 & 3.4e-3 & +0.317 & 6.6e-3 & +0.333 & 1.4e-2 & +0.333 & 2.3e-2 & - & - & - & - \\
 & GARCH & - & -10.5 & - & -13.3 & - & -7.3 & - & -0.6 & - & -7.9 & - & - \\
 & HAR-RV & - & -6.2 & - & -5.2 & - & +7.9 & - & +15.8 & - & +3.1 & - & - \\
\cmidrule(lr){2-14}
 & CE & -12.6 & +6.7 & \textbf{+7.5} & \textbf{-6.4} & -11.8 & -11.8 & -8.8 & -8.5 & -6.4 & -5.0 & -5.7 & 2/8 \\
 & MSE & -16.8 & -5.3 & +5.9 & -13.0 & -6.5 & -11.8 & \textbf{-1.7} & -4.7 & -4.8 & -8.7 & -6.7 & 1/8 \\
 & \cellcolor{green!15} \textbf{VCA} & \cellcolor{green!15} \textbf{-6.0} & \cellcolor{green!15} \textbf{+10.0} & \cellcolor{green!15} +3.6 & \cellcolor{green!15} -19.1 & \cellcolor{green!15} \textbf{+8.0} & \cellcolor{green!15} \textbf{-4.8} & \cellcolor{green!15} -6.2 & \cellcolor{green!15} \textbf{+7.2} & \cellcolor{green!15} \textbf{-0.2} & \cellcolor{green!15} \textbf{-1.7} & \cellcolor{green!15} \textbf{-0.9} & \cellcolor{green!15} \textbf{5/8} \\
\midrule
\multirow{6}{*}{\rotatebox[origin=c]{90}{CSI500}} & Frozen & +0.849 & 4.5e-3 & +0.853 & 9.0e-3 & +0.864 & 2.0e-2 & +0.880 & 3.1e-2 & - & - & - & - \\
 & GARCH & - & -11.4 & - & -13.9 & - & -6.7 & - & -1.4 & - & -8.3 & - & - \\
 & HAR-RV & - & -13.1 & - & -9.0 & - & +4.2 & - & +12.0 & - & -1.5 & - & - \\
\cmidrule(lr){2-14}
 & CE & -3.4 & \textbf{+6.1} & -4.5 & +3.1 & \textbf{-0.6} & -5.9 & -8.3 & -1.2 & -4.2 & +0.5 & -1.9 & 2/8 \\
 & MSE & -4.0 & -5.0 & -2.7 & -8.1 & -3.3 & -5.4 & -5.7 & -2.3 & -3.9 & -5.2 & -4.6 & 0/8 \\
 & \cellcolor{green!15} \textbf{VCA} & \cellcolor{green!15} \textbf{-0.6} & \cellcolor{green!15} -5.7 & \cellcolor{green!15} \textbf{+1.0} & \cellcolor{green!15} \textbf{+4.5} & \cellcolor{green!15} -1.3 & \cellcolor{green!15} \textbf{+3.9} & \cellcolor{green!15} \textbf{-4.9} & \cellcolor{green!15} \textbf{+3.4} & \cellcolor{green!15} \textbf{-1.4} & \cellcolor{green!15} \textbf{+1.5} & \cellcolor{green!15} \textbf{+0.0} & \cellcolor{green!15} \textbf{6/8} \\
\midrule
\multirow{6}{*}{\rotatebox[origin=c]{90}{Crypto}} & Frozen & +0.055 & 1.4e-4 & +0.051 & 3.1e-4 & +0.051 & 6.8e-4 & +0.044 & 1.1e-3 & - & - & - & - \\
 & GARCH & - & -0.7 & - & +9.2 & - & +8.0 & - & +12.6 & - & +7.3 & - & - \\
 & HAR-RV & - & -13.1 & - & -4.1 & - & -0.8 & - & +0.4 & - & -4.4 & - & - \\
\cmidrule(lr){2-14}
 & CE & -12.0 & \textbf{+5.7} & -14.3 & +6.2 & -50.7 & \textbf{+6.2} & -11.3 & +4.5 & -22.1 & \textbf{+5.7} & -8.2 & 2/8 \\
 & MSE & -42.1 & -0.5 & -5.7 & +0.4 & \textbf{-2.7} & -6.6 & +6.5 & +3.6 & -11.0 & -0.8 & -5.9 & 1/8 \\
 & \cellcolor{green!15} \textbf{VCA} & \cellcolor{green!15} \textbf{+3.4} & \cellcolor{green!15} +3.9 & \cellcolor{green!15} \textbf{+7.3} & \cellcolor{green!15} \textbf{+11.0} & \cellcolor{green!15} -6.1 & \cellcolor{green!15} -4.8 & \cellcolor{green!15} \textbf{+10.2} & \cellcolor{green!15} \textbf{+6.0} & \cellcolor{green!15} \textbf{+3.7} & \cellcolor{green!15} +4.0 & \cellcolor{green!15} \textbf{+3.9} & \cellcolor{green!15} \textbf{5/8} \\
\bottomrule
\end{tabular}

}
\end{table}

\paragraph{Metrics.} For the path-wise quantity, we report price RankIC (\textbf{RankIC}), the per-window Spearman correlation between the predicted and realized OHLC price path, averaged over windows. For the volatility side, we report \textbf{$\sigma^2$-MAE}, the mean absolute error of realized variance. Both metrics follow the evaluation protocol of prior work~\citep{kronos2025}. We use RankIC and $\sigma^2$-MAE for our main evaluation, while QLIKE~\citep{patton2011qlike} is reported as a complementary volatility metric in Appendix Table~\ref{tab:mainqlike}. The overall score is
\begin{equation}
\label{eq:score}
\mathrm{Score} \;=\; \tfrac{1}{2}\big(\Delta_\text{RankIC}\% + \Delta_\text{MAE}\%\big),
\end{equation}
where $\Delta$ is relative to the frozen model, with both terms signed so that positive means improvement.
Definitions of RankIC, $\sigma^2$-MAE and additional metrics (QLIKE, price IC, return IC/RankIC, and the ACF$^2$ gap, Eq.~\ref{eq:acfloss}) are in Appendix~\ref{app:metrics}.

\paragraph{Baselines.} We compare our method, \textbf{VCA}, with three methods: (1) the frozen pre-trained model \textbf{Frozen}; (2) plain cross-entropy fine-tuning \textbf{CE}; and (3) \textbf{MSE}, which keeps the identical differentiable AR rollout and replaces only $\mathcal{L}_{\mathrm{ACF}^2}$ with a relative squared error on the rolled-out path,
$\mathcal{L}_{\mathrm{MSE}}=\frac{1}{H}\sum_{t=1}^{H}\big((\hat p_t-p^\star_t)/p^\star_t\big)^2$.
We additionally compare VCA with two classical volatility models, GARCH(1,1) \citep{bollerslev1986garch} and HAR-RV \citep{corsi2009har}, which fit per symbol and forecast at the same windows.

\paragraph{Protocol.}
\label{sec:protocol}
All methods are fully fine-tuned without adapters for the same number of epochs and learning rate over three seeds on 4 V100 GPUs, with window size fixed at $W{=}160$ and horizon $H$ swept over $\{8,16,32,48\}$. Checkpoints are selected by lowest teacher-forced validation cross-entropy for CE, and by lowest aggregate ACF$^2$ gap on autoregressive validation rollouts for VCA and MSE. 
Rollouts follow two sampling conventions: \textsc{fore} ($T{=}0.6$, $N{=}10$) and \textsc{vol} ($T{=}0.9$, $N{=}1$), the latter matching the setting Kronos uses for its own variance metric. The two share a single RankIC, computed once from the \textsc{fore} rollout; only $\sigma^2$-MAE is recomputed per convention, so the \textsc{fore} and \textsc{vol} scores differ through that term alone. 
We report \textsc{fore} in the main tables and \textsc{vol} in the appendix (Table~\ref{tab:volgrid}); QLIKE also uses \textsc{fore}.
More details are in Appendix~\ref{app:imp}.

\begin{figure*}[ht]
\centering
\includegraphics[width=\textwidth]{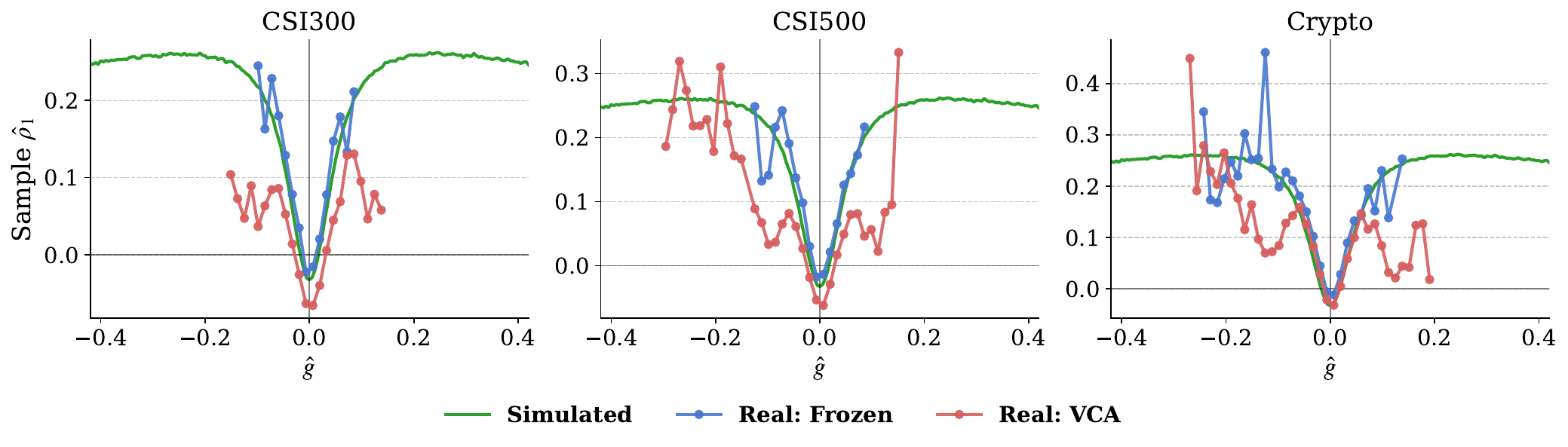}
\caption{Volatility-trend channel, simulated vs.\ real rollouts ($H{=}32$), one panel per dataset. \emph{Simulated} is $\mathbb{E}[\hat\rho_1]$ from $r_t=\sigma_0e^{gt}\varepsilon_t$ (no genuine clustering) as $g$ varies; \emph{Real} shows the mean $\hat\rho_1$ within bins of the per-window trend rate $\hat g$ from the predicted rollout, computed as the log-ratio of second-half to first-half squared returns. We use second-half over first-half so that $\hat g$ has the same sign as $g$, following Proposition~\ref{prop}(b).}
\label{fig:trend_channel}
\end{figure*}

\subsection{Main results}
\label{sec:results}

\paragraph{VCA is the strongest fine-tuning objective in aggregate.}
\label{sec:scaling}
VCA achieves the best average Score among fine-tuned methods and the highest win rate on all three datasets (Table~\ref{tab:main}). 
While aggregate results are reported for comprehensiveness, practical deployment typically targets a fixed horizon. At $H{=}16$, VCA consistently improves RankIC and achieves a higher average Score than the frozen model across datasets.
Across horizons, VCA is also the only fine-tuned method to improve over the frozen model under $\sigma^2$-MAE.
Compared with MSE, which shares the same rollout, VCA's margin isolates the gain to the ACF$^2$ loss itself rather than to rollout-based training in general.
The margin over CE further shows that the ACF$^2$ signal adds information that next-token cross-entropy does not provide. It does so by moving the model further from the pre-trained weights, not closer (Section~\ref{sec:forget}), indicating an additional training signal rather than a regularizer.
Moreover, VCA outperforms both GARCH(1,1) and HAR-RV on CSI500, while each classical baseline is better on one of the other two datasets.

Under \textsc{vol} (Appendix Table~\ref{tab:volgrid}), VCA remains best on CSI300 and Crypto, with CSI500 as the only exception. 
Under QLIKE (Appendix Table~\ref{tab:mainqlike}), VCA shows the largest average improvement on CSI300 and CSI500 but not Crypto, where its residual under-dispersion is penalized more heavily. All fine-tuned methods trail both classical baselines on all three datasets under this metric. These per-symbol-fit models do not provide a price path, and hence no RankIC or Score, or a shared pre-trained backbone. VCA also achieves the lowest average ACF$^2$ gap at each horizon across datasets (Appendix Table~\ref{tab:acf2}).

\paragraph{Evaluation on new backbone.}
We further evaluate VCA on \textbf{Chronos-T5-small} \citep{chronos2024}, a smaller and less-trained backbone, using the same fine-tuning setup at $H{=}16$ and $H{=}32$ across all three datasets (Appendix Table~\ref{tab:chronos_main}). 
VCA produces improvements on Crypto at $H{=}16$ and CSI500 at $H{=}32$, with the latter improving the average Score by $69.1\%$ over the frozen baseline. 
The results are less stable than with Kronos, with rollout instability on CSI300 and Crypto at $H{=}32$, suggesting that VCA benefits more from domain-specific fine-tuning data than from a general-purpose backbone.

\paragraph{Empirical evidence for the volatility-trend channel.}
\label{sec:theory}
Figure~\ref{fig:trend_channel} plots the sample ACF$_1$ of squared returns against the per-window trend rate $\hat g$ for both simulated and real rollouts. In the simulation, a deterministic trend with no true volatility clustering already produces a non-zero ACF when $|\hat g|$ is large, showing that trend alone can mimic the signal optimized by VCA. The Frozen model's real rollouts closely follow the same relationship across all three datasets, confirming that this effect also occurs in real model outputs.
VCA's rollouts instead lie below the curve at matched $\hat g$ values and show substantially more scatter. This suggests that VCA does not simply increase $\hat g$ to exploit the objective, although the relationship does not fully explain the observed instability.

\begin{table}[t]
\centering
\small
\caption{Regime-matched definition. vol($\cdot$) is the standard deviation of log returns on close price (Appendix~\ref{app:metrics}). test/first and test/last are the ratios used to pick the matched slice (bold: closer to one); val/first and val/last are the same ratios computed from the validation period.The shaded column identifies whether the \textit{first} or \textit{last} slice is selected.}
\label{tab:matchrule}
\resizebox{\textwidth}{!}{
\begin{tabular}{llccccccccc}
\toprule
Dataset & Slice & vol(first) & vol(last) & vol(test) & vol(val) & test/first & test/last & val/first & val/last & matched \\
\midrule
\multirow{2}{*}{CSI300} & $10\%$ & 0.0298 & 0.0192 & 0.0253 & 0.0241 & \textbf{0.85$\times$} & 1.32$\times$ & \textbf{0.81$\times$} & 1.26$\times$ & \cellcolor{green!15} first \\
 & $1\%$ & 0.0287 & 0.0182 & 0.0253 & 0.0241 & \textbf{0.88$\times$} & 1.39$\times$ & \textbf{0.84$\times$} & 1.33$\times$ & \cellcolor{green!15} first \\
\midrule
\multirow{2}{*}{CSI500} & $10\%$ & 0.0299 & 0.0213 & 0.0282 & 0.0267 & \textbf{0.94$\times$} & 1.32$\times$ & \textbf{0.89$\times$} & 1.25$\times$ & \cellcolor{green!15} first \\
 & $1\%$ & 0.0235 & 0.0196 & 0.0282 & 0.0267 & \textbf{1.20$\times$} & 1.44$\times$ & \textbf{1.13$\times$} & 1.36$\times$ & \cellcolor{green!15} first \\
\midrule
\multirow{2}{*}{Crypto} & $10\%$ & 0.0105 & 0.0042 & 0.0051 & 0.0050 & 0.48$\times$ & \textbf{1.21$\times$} & 0.47$\times$ & \textbf{1.18$\times$} & \cellcolor{green!15} last \\
 & $1\%$ & 0.0133 & 0.0051 & 0.0051 & 0.0050 & 0.38$\times$ & \textbf{0.99$\times$} & 0.38$\times$ & \textbf{0.97$\times$} & \cellcolor{green!15} last \\
\bottomrule
\end{tabular}

}
\end{table}
\begin{figure}[t]
\centering
\includegraphics[width=\textwidth]{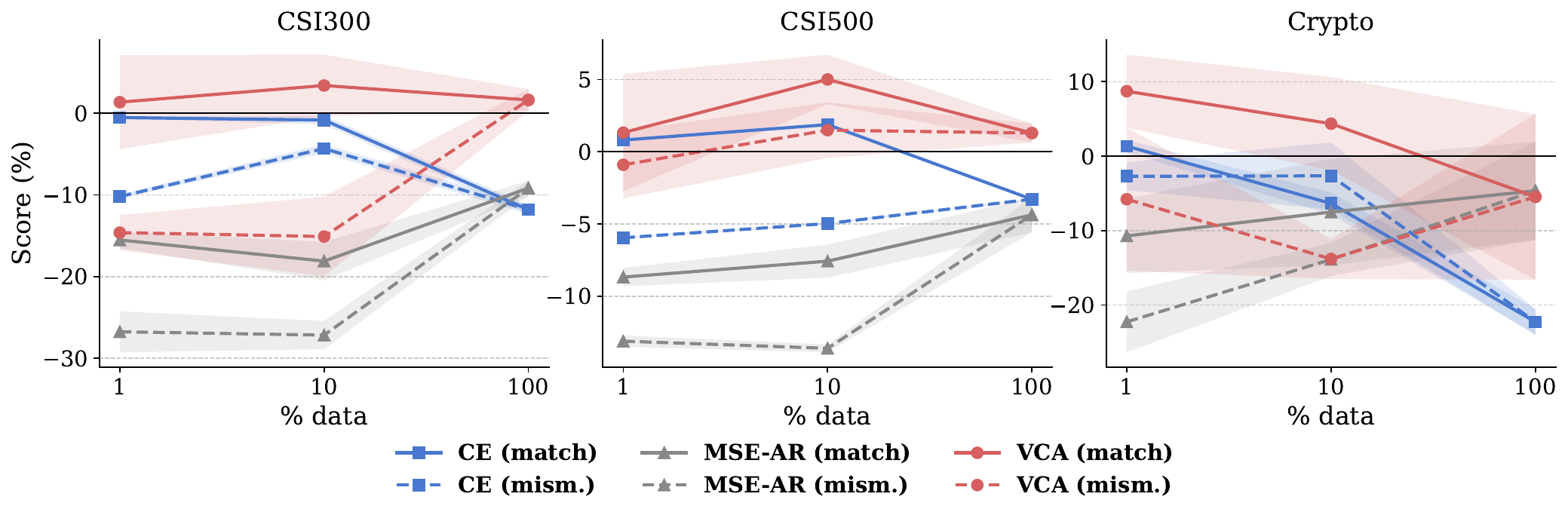}
\caption{Data efficiency across all three methods, Score against the frozen model ($H{=}32$, \textsc{fore}). Solid lines are the \textit{regime-matched} slice and dashed the \textit{mismatched} one, with bands of one standard deviation over 3 seeds. The regime-matched slices are identified in Table~\ref{tab:matchrule}. Detailed values are in Appendix Table~\ref{tab:dataeff}.}
\label{fig:dataeff}
\end{figure}

\subsection{Ablations}
\label{sec:ablations}

\paragraph{A regime-matched slice makes fine-tuning data-efficient.}
\label{sec:dataeff}

Fine-tuning on a chronological slice (first or last) of the training data can outperform full-data fine-tuning when the slice's volatility regime matches the test period, suggesting an \emph{alignment} effect rather than a purely \emph{recency} effect. Table~\ref{tab:matchrule} determines the matched slice without evaluating the test period. The validation-based rule picks the same slice as the test-based rule in all six (dataset, slice-size) cells, showing that the selection can be made before deployment rather than after observing the test period. 
Volatility $\mathrm{vol}(\cdot)$ is more discriminative than an ACF-based statistic, with raw lag-1 ACF of squared returns disagreeing across slice sizes on CSI300 (Appendix Table~\ref{tab:matchrule_acf1}), consistent with noisier ACF estimates at this sample size (Section~\ref{sec:mainmethod}).
The matched slice outperforms full-data fine-tuning across datasets and evaluation conventions at the tested slice sizes, while full-data fine-tuning is negative on Crypto (Figure~\ref{fig:dataeff}). 
Mismatch is costly on Crypto and CSI300 but not on CSI500, so regime matching should be viewed as a useful selection rule rather than a universal penalty for mismatch.

This data-efficiency effect is specific to $H{=}32$ and does not transfer to $H{=}16$, where small-data fine-tuning remains substantially below the frozen model until the full dataset is used (Appendix Figure~\ref{fig:dataeff16}). The degradation reflects a drop in RankIC and affects CE as well as VCA, suggesting that broader data coverage is more important at shorter horizons, while regime alignment becomes more useful at longer horizons. We therefore view the two factors as complementary rather than treating regime alignment as a general substitute for data volume.

\begin{figure}[t]
\centering
\includegraphics[width=\textwidth]{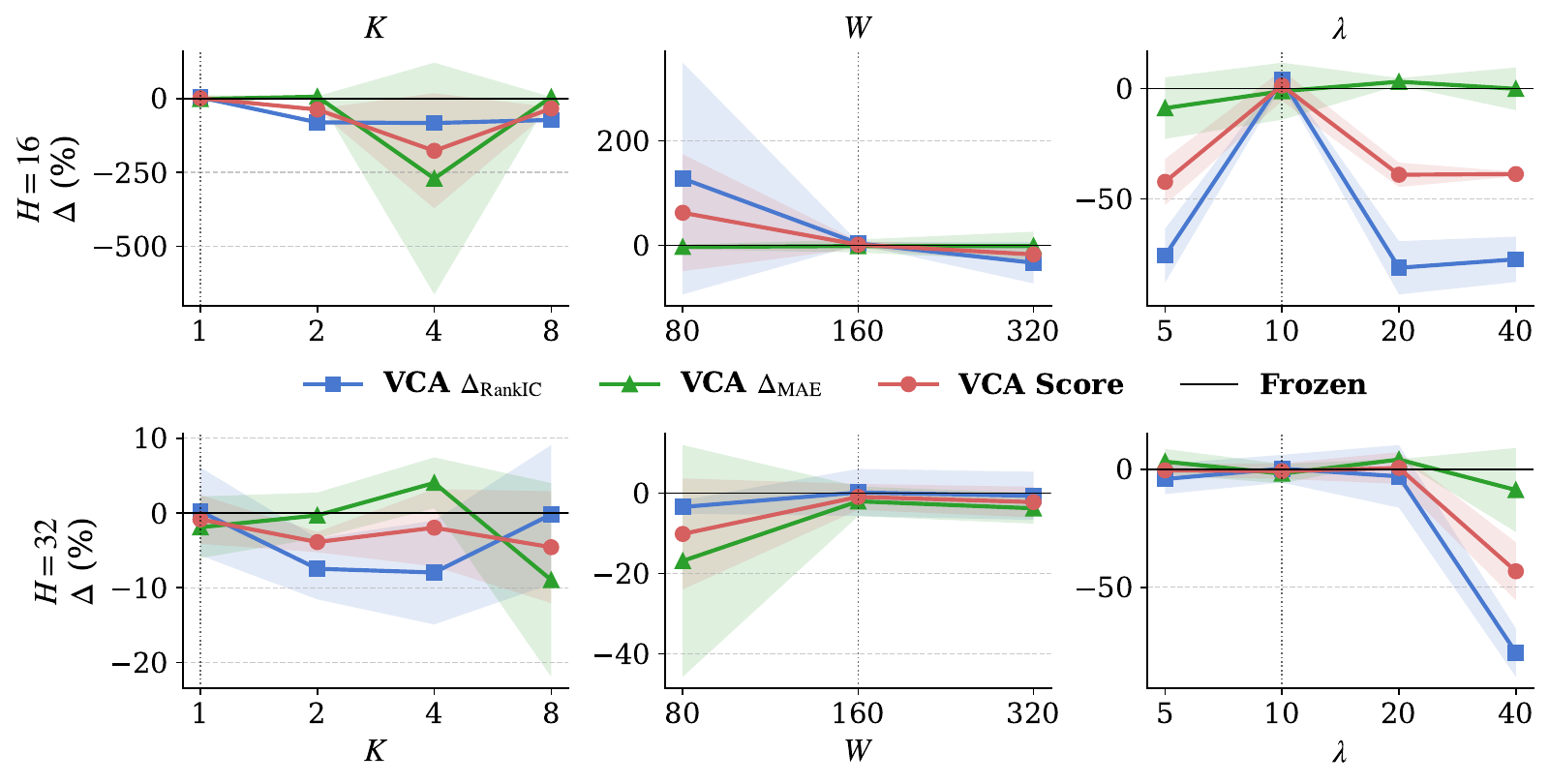}
\caption{Ablation sweeps under \textsc{fore}, mean over the three datasets, band is one standard deviation of the per-dataset delta. $\Delta_\text{RankIC}$ and $\Delta_\text{MAE}$ are the two halves Score averages; the black line at 0 is the frozen backbone. Detailed values are in Appendix Tables~\ref{tab:abl} and~\ref{tab:abl_h32}.}
\label{fig:kwlambda}
\end{figure}

\paragraph{Does adaptation forget?}
\label{sec:forget}
The regime-matched slice matches full-data fine-tuning on the target, raising the question of how much additional weight drift the unused data introduces. We measure relative weight drift, $|\theta-\theta_0|_2/|\theta_0|_2$ from the frozen model $\theta_0$, as a proxy for how much a model \emph{could} have forgotten  (Appendix Table~\ref{tab:drift}). At full data and $H{=}32$, VCA produces larger weight drift than CE, with a smaller difference at $H{=}16$. The regime-matched 10\% slice reduces VCA's weight drift at both horizons while matching or outperforming full-data fine-tuning on the target. However, its drift remains comparable to or higher than CE, indicating that regime matching reduces, but does not eliminate, the additional drift associated with VCA. At $H{=}16$, full-data drift is larger in five of the six method--dataset combinations, with VCA on CSI300 as the only exception.

\paragraph{Effect of lag count $K$.}
The lag count has a stronger effect at $H{=}16$. The choice $K{=}1$ produces a sharp peak, while larger lag counts substantially reduce RankIC and increase seed-to-seed variation. The same ordering remains at $H{=}32$, but the performance gap is much smaller across the sweep, with $K{=}1$ still leading on all three datasets (Figure~\ref{fig:kwlambda}-left).

\paragraph{Effect of lookback $W$.}
The lookback shows a similar horizon-dependent pattern. At $H{=}16$, raw RankIC varies sharply with $W$ and declines on both sides of the best-performing setting ($W{=}160$). The $\Delta$-based Score in Figure~\ref{fig:kwlambda} makes this contrast appear even larger because it is measured relative to the frozen backbone at the same $W$, whose performance is much lower with shorter context. At $H{=}32$, this dependence largely disappears, with raw RankIC remaining close across $W$ values.

\paragraph{Effect of weight $\lambda$.}
The ACF$^2$ weight behaves differently from the other two hyperparameters. At $H{=}16$, the recipe value gives a narrow peak: reducing $\lambda$ weakens the auxiliary signal mildly, whereas increasing it causes much larger RankIC losses before the effect saturates. This sharp sensitivity persists at $H{=}32$, where the largest tested weight again causes severe degradation. Thus, unlike $W$ and $K$, $\lambda$ remains sensitive across horizons.

\section{Conclusion}
\label{sec:conclusion}

We introduce \emph{Volatility-Clustering Adaptation (VCA)}, which augments next-token prediction with a differentiable signal for the autocorrelation of squared returns over autoregressive rollouts. Across three asset sets and two evaluation conventions, VCA delivers the best average performance, primarily reflected in variance-related metrics. We further find that volatility-matched data can improve data efficiency at longer horizons, while broader data coverage remains important at shorter horizons. 
These findings highlight the value of adapting time-series foundation models with temporal objectives rather than relying on token-level prediction.
Limitations are discussed in Appendix~\ref{sec:limitations}.

\subsection*{AI use statement}
We used generative AI tools solely to support writing, including language polishing and clarity, not to
develop ideas, design methods, or run experiments. All AI-assisted text was reviewed and verified by the authors, who take full responsibility for the final content.

\subsection*{Ethics statement}
VCA is a research recipe, not a validated trading system, so deploying it on
live capital without independently verifying stability could turn a research-scale gain into a real
loss. We use only public data and open pre-trained backbones, with no human subjects.

\subsection*{Reproducibility statement}
Datasets and implementation details are in Section~\ref{sec:experiment}, Table~\ref{tab:data} and Appendix Table~\ref{tab:hyper}. Proof of Proposition~\ref{prop} is in Appendix~\ref{app:theory}. 
We will release the code upon publication.

\bibliography{iclr2027_conference}
\bibliographystyle{iclr2027_conference}

\clearpage
\appendix

\section*{Appendix}

\section{Additional Implementation Details}\label{app:imp}

We summarize dataset details in Table~\ref{tab:data} and hyperparameters in Table~\ref{tab:hyper}.
\paragraph{GARCH/HAR-RV implementation.}
\label{app:garch}

Both baselines are fit once per symbol on the combined training and validation split, as they have no checkpoint to select and therefore require no separate validation-based calibration.
Each baseline then forecasts at exactly the same strided test windows Kronos uses (same lookback$=$160, stride, and test period per dataset). 
GARCH(1,1) is fit by maximum likelihood estimation; HAR-RV regresses realized variance on its own lags
$\{1,5,22\}$ by ordinary least squares. Lags are bar-level, not calendar day/week/month, matching how QLIKE and $\sigma^2$ are defined throughout this paper (Appendix~\ref{app:metrics}). 

\section{Proof of Proposition~\ref{prop}}
\label{app:theory}

\begin{proof}
Write
\begin{equation}
\label{eq:Xg}
X_t^{(g)}:=r_t^2=\sigma_0^2e^{2gt}\varepsilon_t^2,
\qquad t=1,\dots,n,
\end{equation}
let $X^{(g)}:=X_{1:n}^{(g)}$, and let $\mathcal Rx_{1:n}:=(x_n,\dots,x_1)$
denote time reversal. Since $\varepsilon_t^2$ is non-degenerate, the
denominator of Eq.~\ref{eq:rho} evaluated at $X^{(g)}$ is a.s.\ positive,
so $\hat\rho_k(X^{(g)})$ is a.s.\ well defined.

We first record two elementary invariances of Eq.~\ref{eq:rho}, valid for
every $x_{1:n}$ with positive sample variance, every
$k\in\{1,\dots,n-1\}$ and every $c\neq0$:
\begin{align}
\hat\rho_k(\mathcal Rx_{1:n})&=\hat\rho_k(x_{1:n}),
\label{eq:rev-invariance}\\[2pt]
\hat\rho_k(c\,x_{1:n})&=\hat\rho_k(x_{1:n}).
\label{eq:scale-invariance}
\end{align}
Eq.~\ref{eq:rev-invariance} holds because reversal is a permutation of
the entries, hence leaves $\bar x$ and $\sum_t(x_t-\bar x)^2$ unchanged,
while in the numerator the substitution $s=n+1-t-k$ is a bijection of
$\{1,\dots,n-k\}$ onto itself carrying the pair
$\bigl((\mathcal Rx)_t,(\mathcal Rx)_{t+k}\bigr)=(x_{n+1-t},x_{n+1-t-k})$
to $(x_{s+k},x_s)$, so that reversal merely swaps the two coordinates of
each lag-$k$ product. Eq.~\ref{eq:scale-invariance} holds because
rescaling by $c$ multiplies the numerator and the denominator of
Eq.~\ref{eq:rho} by the same factor $c^2$.

\medskip
\noindent\emph{(a)} Reversing Eq.~\ref{eq:Xg} gives the pathwise identity
\begin{equation}
\label{eq:reverse-Xg}
\bigl(\mathcal RX^{(g)}\bigr)_t
=X_{n+1-t}^{(g)}
=e^{2g(n+1)}\,\sigma_0^2e^{-2gt}\varepsilon_{n+1-t}^2,
\qquad t=1,\dots,n,
\end{equation}
that is, $\mathcal RX^{(g)}=e^{2g(n+1)}Y$ with
$Y_t:=\sigma_0^2e^{-2gt}\varepsilon_{n+1-t}^2$ and $e^{2g(n+1)}>0$.
Combining Eq.~\ref{eq:reverse-Xg} with
Eqs.~\ref{eq:rev-invariance}--\ref{eq:scale-invariance},
\begin{equation}
\label{eq:acf-chain}
\hat\rho_k\!\left(X^{(g)}\right)
=\hat\rho_k\!\left(\mathcal RX^{(g)}\right)
=\hat\rho_k\!\left(Y\right)
\qquad\text{a.s., for every }k.
\end{equation}
Since the $\varepsilon_t$ are i.i.d., they are exchangeable, so
\begin{equation}
\label{eq:exchangeability}
(\varepsilon_n^2,\ldots,\varepsilon_1^2)
\overset{d}{=}(\varepsilon_1^2,\ldots,\varepsilon_n^2);
\end{equation}
applying the deterministic map
$(u_1,\dots,u_n)\mapsto(\sigma_0^2e^{-2gt}u_t)_{t=1}^n$ to both sides of
Eq.~\ref{eq:exchangeability} yields $Y\overset{d}{=}X^{(-g)}$ as random
vectors in $\mathbb R^n$. As $\hat\rho_k$ is measurable,
Eq.~\ref{eq:acf-chain} gives
\begin{equation}
\label{eq:acf-symmetry}
\hat\rho_k\!\left(X^{(g)}\right)
\overset{d}{=}
\hat\rho_k\!\left(X^{(-g)}\right),
\end{equation}
jointly over any finite set of lags, proving (a).

\medskip
\noindent\emph{(b)} Since $\mathbb E[\varepsilon_t^2]=1$,
\begin{equation}
\label{eq:expected-energy}
\mathbb E\!\left[\sum_{t=1}^n r_t^2\right]
=\sigma_0^2\sum_{t=1}^n e^{2gt}=:S_n(g),
\end{equation}
which is strictly increasing in $g$ because
\begin{equation}
\label{eq:monotonicity}
S_n'(g)=2\sigma_0^2\sum_{t=1}^n t\,e^{2gt}>0
\qquad\text{for all }g\in\mathbb R.
\end{equation}
Finally, setting $q:=e^{2g}>0$, the substitution $s=n+1-t$ gives
\begin{equation}
\label{eq:reverse-geometric}
\sum_{t=1}^n q^{-t}
=\sum_{s=1}^n q^{-(n+1-s)}
=q^{-(n+1)}\sum_{s=1}^n q^{s},
\end{equation}
and hence
\begin{equation}
\label{eq:energy-ratio}
\frac{\mathbb E\bigl[\sum_t r_t^2\,;\,+g\bigr]}
     {\mathbb E\bigl[\sum_t r_t^2\,;\,-g\bigr]}
=\frac{\sum_{t=1}^n q^{t}}{\sum_{t=1}^n q^{-t}}
=q^{\,n+1}=e^{2g(n+1)},
\end{equation}
proving (b).
\end{proof}

\section{Metric definitions}
\label{app:metrics}

Notation: $p_{1:H}$ is a close-price path over the $H$-bar test window, $p_0$ the anchor bar,
$r_t=\log(p_t/p_{t-1})$ the log-return, $\hat{\phantom{p}}$ predictions, $\star$ realized values,
and $\mathcal{W}$ the set of (symbol, test-window) pairs.

\paragraph{Rank IC.}
Spearman correlation between predicted and realized paths over the $H$ steps, averaged across the
four OHLC channels and all windows:
\begin{equation}
\label{eq:rankic}
\mathrm{RankIC} \;=\; \frac{1}{|\mathcal{W}|}\sum_{w\in\mathcal{W}} \frac{1}{4}\sum_{c\in\{O,H,L,C\}}
\mathrm{Spearman}\big(\hat p^{(c)}_{1:H,w},\, p^{\star(c)}_{1:H,w}\big).
\end{equation}
Replacing Spearman with Pearson correlation creates \emph{price IC}.

\paragraph{$\sigma^2$-MAE.}
Realized variance of a path is the sum of squared log-returns, with no annualization or square root:
\begin{equation}
\sigma^2(p_{1:H}) \;=\; \sum_{t=1}^{H-1}\big(\log p_{t+1}-\log p_t\big)^2 ,
\end{equation}
and the metric is its absolute error against the realized value, with the prediction averaged over
the $N$ test-time rollouts:
\begin{equation}
\sigma^2\text{-}\mathrm{MAE} \;=\; \frac{1}{|\mathcal{W}|}\sum_{w\in\mathcal{W}}
\big|\sigma^2(\hat p_{1:H,w})-\sigma^2(p^\star_{1:H,w})\big| .
\end{equation}

\paragraph{Parkinson volatility estimator.}
Used only in Table~\ref{tab:headroom}, alongside the same $\sigma^2(\cdot)$ already defined
above (``Vol.\ level (close)''), as an alternative realized-variance proxy from each
bar's intra-bar high $H_t$ and low $L_t$ rather than close-to-close returns (``Vol.\
level (Parkinson)''):
\begin{equation}
\sigma^2_{\mathrm{park}}(p_{1:H}) \;=\; \sum_{t=1}^{H}\Big(\log\tfrac{H_t}{L_t}\Big)^2 .
\end{equation}
For both proxies, Table~\ref{tab:headroom} scores the model with Model/Naive: the model's squared
error against always predicting the training-set mean $\overline{\sigma^{2\star}}$, equivalently
$1-R^2$:
\begin{equation}
\mathrm{Model/Naive} \;=\; \frac{\sum_{w\in\mathcal{W}}\big(\hat\sigma^2_w-\sigma^{2\star}_w\big)^2}
{\sum_{w\in\mathcal{W}}\big(\sigma^{2\star}_w-\overline{\sigma^{2\star}}\big)^2}
\;=\; 1-R^2 .
\end{equation}

\paragraph{QLIKE.}
The standard robust loss for a noisy realized-variance proxy \citep{patton2011qlike}: writing
$\sigma^{2\star}=\sigma^2(p^\star_{1:H,w})$ and $\hat\sigma^2=\sigma^2(\hat p_{1:H,w})$ for the realized
and predicted variance of a window (prediction again averaged over the $N$ test-time rollouts),
\begin{equation}
\mathrm{QLIKE} \;=\; \frac{1}{|\mathcal{W}|}\sum_{w\in\mathcal{W}}
\left(\frac{\sigma^{2\star}_w}{\hat\sigma^2_w} - \log\frac{\sigma^{2\star}_w}{\hat\sigma^2_w} - 1\right),
\end{equation}
which is $0$ at $\hat\sigma^2_w=\sigma^{2\star}_w$ and grows asymmetrically faster as
$\hat\sigma^2_w\to 0$ (under-prediction) than as $\hat\sigma^2_w\to\infty$ (over-prediction). As with
$\sigma^2$-MAE, both terms are winsorized at the $99$th percentile per (dataset, convention) cell
before averaging.

\paragraph{Regime-match volatility, $\mathrm{vol}(\cdot)$.}
Used to define the regime-matched slice (Table~\ref{tab:matchrule}, Section~\ref{sec:dataeff}).
For a period $\mathcal{P}$ (train-first-$10\%$, train-last-$10\%$, validation, or test) let
$r_t^{(s)}=\log(p_t^{(s)}/p_{t-1}^{(s)})$ be the close-price log-return of symbol $s$ at bar $t$, pooled
over all symbols and bars in $\mathcal{P}$, with mean $\bar r_{\mathcal{P}}$. Then
\begin{equation}
\mathrm{vol}(\mathcal{P}) \;=\; \sqrt{\frac{1}{|\mathcal{P}|}\sum_{(s,t)\in\mathcal{P}}
\big(r_t^{(s)}-\bar r_{\mathcal{P}}\big)^2},
\end{equation}
where $\bar r_{\mathcal{P}}$ is the mean return over the same observations. We do not annualize this quantity because only within-dataset ratios of $\mathrm{vol}(\cdot)$ are used, so the common annualization factor would cancel.

\paragraph{Return IC / Return RankIC.}
The end-of-window simple return is $\hat r_{s,t}=\hat p_{H,s,t}/p_{0,s,t}-1$ for the prediction and
$r^\star_{s,t}=p^\star_{H,s,t}/p_{0,s,t}-1$ for the realized path. 
Both metrics correlate these returns \emph{cross-sectionally} over the symbols $S_t$
with a valid window anchored at test timestamp $t$ (a calendar date for daily equity bars,
a 15-min bar for crypto), then average over timestamps:
\begin{align}
\mathrm{IC}_t &= \mathrm{Pearson}\big(\{\hat r_{s,t}\}_{s\in S_t},\,\{r^\star_{s,t}\}_{s\in S_t}\big),
& \mathrm{IC} &= \frac{1}{|T|}\sum_{t\in T}\mathrm{IC}_t ,\\
\mathrm{RankIC}_t &= \mathrm{Spearman}\big(\{\hat r_{s,t}\}_{s\in S_t},\,\{r^\star_{s,t}\}_{s\in S_t}\big),
& \mathrm{RankIC} &= \frac{1}{|T|}\sum_{t\in T}\mathrm{RankIC}_t .
\end{align}

\paragraph{ACF$^2$ gap.}
Let $\hat\rho_{k,w}$ and $\rho^\star_{k,w}$ be the lag-$k$ sample autocorrelations within window $w$ (Eq.~\ref{eq:rho}).
The gap squares the per-lag discrepancy, averages over the $K$ lags, then over windows:
\begin{equation}
\mathrm{ACF}^2\text{-gap} \;=\; \frac{1}{|\mathcal{W}|}\sum_{w\in\mathcal{W}} \frac{1}{K}\sum_{k=1}^{K}
\big(\hat\rho_{k,w}-\rho^\star_{k,w}\big)^2 .
\end{equation}
The training loss differs, where the ACFs are averaged over a batch before the squared difference is taken.

\section{Limitations}
\label{sec:limitations}

\textbf{First}, VCA's advantage appears conditional on a well-pre-trained backbone, and the gains largely disappear on Chronos-t5-small, whose zero-shot RankIC is an order of magnitude lower.

\textbf{Second}, although VCA is strongest on average, its gains are not consistent across benchmarks, and the scaling trend remains unclear, leaving room for future work.

\begin{table}[ht]
\centering
\caption{Dataset details. Equity universes are the survivorship-free point-in-time membership union sourced
via Qlib, and Crypto is from the Binance API.}
\label{tab:data}
\resizebox{\textwidth}{!}{
\begin{tabular}{llcclc}
\toprule
{Dataset} & {Source} & {Freq.} & \#{Symbols} & {Splits} & \#{Samples} ($n_\text{tr}/n_\text{val}/n_\text{te}$) \\
\midrule
Crypto top-20 & Binance API & 15-min & 20 & 2021--23 / 2024H1 / 2024H2--25 & 63514 / 10860 / 19308 \\
CSI300        & Qlib & daily & 515 & 2015--17 / 2018 / 2019--20 & 3991 / 1836 / 3557 \\
CSI500        & Qlib & daily & 937 & 2015--17 / 2018 / 2019--20 & 6246 / 2991 / 5749 \\
\bottomrule
\end{tabular}
}
\end{table}

\begin{table}[ht]
\centering
\caption{Hyperparameter details.}
\label{tab:hyper}
\begin{tabular}{lll}
\toprule
& Setting & Value \\
\midrule
\multirow{2}{*}{Backbone}
 & Model & Kronos-base ($102$M), decoder-only \\
 & Adaptation & full fine-tuning, no adapters \\
\midrule
\multirow{7}{*}{Optimization}
 & Epochs & $10$ \\
 & Learning rate & $5\times10^{-5}$ \\
 & Optimizer & $\beta=(0.9,\,0.95)$, AdamW~\citep{loshchilov2017decoupled} \\
 & Weight decay & $0.1$ \\
 & Gradient clipping & $3.0$ \\
 & Seeds & $\{1,2,3\}$ \\
 & Effective batch & $32$ \\
\midrule
\multirow{5}{*}{Data}
 & Lookback $W$ & $160$ bars \\
 & Horizon $H$ & $\{8,16,32,48\}$ \\
 & Max context & $512$ \\
 & Return clipping & $\pm5.0\sigma$ \\
 & Eval stride & $16$ bars \\
\midrule
\multirow{6}{*}{VCA objective}
 & Weight $\lambda_{\mathrm{acf}}$ & $10$ \\
 & Lags $K$ & $1$ \\
 & Lag weights & $w_k=e^{-k/\gamma}$, $\gamma=2.0$ \\
 & Target form & aggregate (batch-mean ACF matched) \\
 & Rollout & fully differentiable autoregressive \\
 & Gumbel-softmax & straight-through, $\tau$ decay $0.5\!\to\!0.1$ \\
\midrule
\multirow{2}{*}{Selection}
 & CE & lowest teacher-forced validation CE \\
 & VCA, MSE & lowest AR-validation aggregate ACF$^2$ gap \\
\midrule
\multirow{3}{*}{Evaluation}
 & \textsc{fore} convention & $T{=}0.6$, $N{=}10$ rollouts \\
 & \textsc{vol} convention & $T{=}0.9$, $N{=}1$ rollout \\
 & ACF lags (metric) & $15$ \\
\bottomrule
\end{tabular}

\end{table}

\begin{table}[h]
\centering
\caption{Headroom on the frozen backbone at $H{=}16$, model error against a naive baseline. 
ACF$^2$: model gap vs.\ a $\hat\rho{=}0$ (no-clustering) floor. 
Vol.\ level (close/Parkinson): model realized-variance error vs.\ always predicting the training-set mean. Detailed metrics are in Appendix~\ref{app:metrics}.}
\label{tab:headroom}
\begin{tabular}{lcc|cc|cc}
\toprule
& \multicolumn{2}{c}{ACF$^2$ gap} & \multicolumn{2}{c}{Vol.\ level (close)} & \multicolumn{2}{c}{Vol.\ level (Parkinson)} \\
\cmidrule(lr){2-3}\cmidrule(lr){4-5}\cmidrule(lr){6-7}
Dataset & gap & Model/Naive & $R^2$ & Model/Naive & $R^2$ & Model/Naive \\
\midrule
CSI300 & \cellcolor{green!15} 0.0392 & \cellcolor{green!15} 1.96$\times$ & $-4.62$ & 5.62$\times$ & $-0.70$ & 1.70$\times$ \\
CSI500 & \cellcolor{green!15} 0.0394 & \cellcolor{green!15} 1.95$\times$ & $-3.17$ & 4.17$\times$ & $-1.60$ & 2.60$\times$ \\
Crypto & \cellcolor{green!15} 0.0417 & \cellcolor{green!15} 1.93$\times$ & $+0.02$ & 0.98$\times$ & $+0.17$ & 0.83$\times$ \\
\bottomrule
\end{tabular}

\end{table}

\begin{table}[h]
\centering
\caption{Secondary comparison, QLIKE, computed under the \textsc{fore} sampling convention. Frozen's row is the raw QLIKE; every other row is a relative delta against Frozen, $\Delta_\text{QLIKE}\%$ signed so positive means improvement, with the last column averaging over horizons. Bold marks the best fine-tuned method in each column, and \textbf{WR} is the win rate over the 4 per-horizon columns to its left.}
\label{tab:mainqlike}
\begin{tabular}{clrrrr|rc}
\toprule
 &  & $H{=}8$ & $H{=}16$ & $H{=}32$ & $H{=}48$ & average over $H$ & \\
\cmidrule(lr){3-6} \cmidrule(lr){7-7}
DS & Method & $\Delta_\text{QLIKE}$ & $\Delta_\text{QLIKE}$ & $\Delta_\text{QLIKE}$ & $\Delta_\text{QLIKE}$ & $\Delta_\text{QLIKE}$ & WR \\
\midrule
\multirow{6}{*}{\rotatebox[origin=c]{90}{CSI300}} & Frozen & 3.13 & 3.36 & 4.65 & 6.04 & - & - \\
 & GARCH & +84.3 & +88.5 & +92.6 & +94.8 & +90.1 & - \\
 & HAR-RV & +84.3 & +88.6 & +93.2 & +95.3 & +90.3 & - \\
\cmidrule(lr){2-8}
 & CE & -27.0 & -49.4 & -90.5 & -69.9 & -59.2 & 0/4 \\
 & MSE & -257.1 & -128.4 & -97.1 & -14.7 & -124.3 & 0/4 \\
 & \cellcolor{green!15} \textbf{VCA} & \cellcolor{green!15} \textbf{-3.7} & \cellcolor{green!15} \textbf{+7.8} & \cellcolor{green!15} \textbf{-11.4} & \cellcolor{green!15} \textbf{+46.3} & \cellcolor{green!15} \textbf{+9.8} & \cellcolor{green!15} \textbf{4/4} \\
\midrule
\multirow{6}{*}{\rotatebox[origin=c]{90}{CSI500}} & Frozen & 3.32 & 3.66 & 4.91 & 6.35 & - & - \\
 & GARCH & +84.1 & +88.5 & +92.3 & +94.4 & +89.8 & - \\
 & HAR-RV & +83.1 & +87.9 & +92.6 & +94.8 & +89.6 & - \\
\cmidrule(lr){2-8}
 & CE & -3.4 & \textbf{+5.5} & -52.9 & -7.2 & -14.5 & 1/4 \\
 & MSE & -184.9 & -89.7 & -65.7 & -16.6 & -89.2 & 0/4 \\
 & \cellcolor{green!15} \textbf{VCA} & \cellcolor{green!15} \textbf{+7.4} & \cellcolor{green!15} +4.2 & \cellcolor{green!15} \textbf{+18.6} & \cellcolor{green!15} \textbf{+28.9} & \cellcolor{green!15} \textbf{+14.8} & \cellcolor{green!15} \textbf{3/4} \\
\midrule
\multirow{6}{*}{\rotatebox[origin=c]{90}{Crypto}} & Frozen & 4.84 & 6.77 & 10.14 & 11.72 & - & - \\
 & GARCH & +91.5 & +94.8 & +96.9 & +97.3 & +95.1 & - \\
 & HAR-RV & +88.3 & +92.7 & +95.5 & +96.1 & +93.1 & - \\
\cmidrule(lr){2-8}
 & \cellcolor{green!15} \textbf{CE} & \cellcolor{green!15} \textbf{+47.5} & \cellcolor{green!15} +47.8 & \cellcolor{green!15} \textbf{+52.6} & \cellcolor{green!15} \textbf{+47.7} & \cellcolor{green!15} \textbf{+48.9} & \cellcolor{green!15} \textbf{3/4} \\
 & MSE & -5.6 & -1.3 & -787.4 & +14.8 & -194.9 & 0/4 \\
 & VCA & +27.3 & \textbf{+61.8} & +38.3 & +0.2 & +31.9 & 1/4 \\
\bottomrule
\end{tabular}

\end{table}

\begin{table}[h]
\centering
\caption{Same layout as Table~\ref{tab:main}, with $\sigma^2$-MAE recomputed under \textsc{vol} in place of \textsc{fore}. RankIC is unchanged from Table~\ref{tab:main}, since it is computed once from the \textsc{fore} rollout and shared across conventions; only $\Delta_\text{MAE}$ and the resulting Score differ. Frozen row is the same as in Table~\ref{tab:main}.}
\label{tab:volgrid}
\setlength{\tabcolsep}{2pt}
\resizebox{\textwidth}{!}{%
\begin{tabular}{clrrrrrrrr|rrrc}
\toprule
 &  & \multicolumn{2}{c}{$H{=}8$} & \multicolumn{2}{c}{$H{=}16$} & \multicolumn{2}{c}{$H{=}32$} & \multicolumn{2}{c}{$H{=}48$} & \multicolumn{3}{c}{average over $H$} & \\
\cmidrule(lr){3-4} \cmidrule(lr){5-6} \cmidrule(lr){7-8} \cmidrule(lr){9-10} \cmidrule(lr){11-13}
DS & Method & $\Delta_\text{RankIC}$ & $\Delta_\text{MAE}$ & $\Delta_\text{RankIC}$ & $\Delta_\text{MAE}$ & $\Delta_\text{RankIC}$ & $\Delta_\text{MAE}$ & $\Delta_\text{RankIC}$ & $\Delta_\text{MAE}$ & $\Delta_\text{RankIC}$ & $\Delta_\text{MAE}$ & Score$_{\text{MAE}}$ & WR \\
\midrule
\multirow{6}{*}{\rotatebox[origin=c]{90}{CSI300}} & Frozen & +0.291 & 3.4e-3 & +0.317 & 6.1e-3 & +0.333 & 1.3e-2 & +0.333 & 2.0e-2 & - & - & - & - \\
 & GARCH & - & -11.1 & - & -22.4 & - & -22.1 & - & -17.3 & - & -18.2 & - & - \\
 & HAR-RV & - & -6.8 & - & -13.7 & - & -4.7 & - & +1.8 & - & -5.8 & - & - \\
\cmidrule(lr){2-14}
 & CE & -12.6 & +8.7 & \textbf{+7.5} & \textbf{-1.0} & -11.8 & -8.3 & -8.8 & -9.8 & -6.4 & \textbf{-2.6} & -4.5 & 2/8 \\
 & MSE & -16.8 & -1.9 & +5.9 & -15.6 & -6.5 & -19.7 & \textbf{-1.7} & -12.8 & -4.8 & -12.5 & -8.7 & 1/8 \\
 & \cellcolor{green!15} \textbf{VCA} & \cellcolor{green!15} \textbf{-6.0} & \cellcolor{green!15} \textbf{+8.8} & \cellcolor{green!15} +3.6 & \cellcolor{green!15} -32.6 & \cellcolor{green!15} \textbf{+8.0} & \cellcolor{green!15} \textbf{-2.6} & \cellcolor{green!15} -6.2 & \cellcolor{green!15} \textbf{-2.8} & \cellcolor{green!15} \textbf{-0.2} & \cellcolor{green!15} -7.3 & \cellcolor{green!15} \textbf{-3.8} & \cellcolor{green!15} \textbf{5/8} \\
\midrule
\multirow{6}{*}{\rotatebox[origin=c]{90}{CSI500}} & Frozen & +0.849 & 4.5e-3 & +0.853 & 8.4e-3 & +0.864 & 1.7e-2 & +0.880 & 2.7e-2 & - & - & - & - \\
 & GARCH & - & -10.9 & - & -22.1 & - & -22.0 & - & -17.9 & - & -18.2 & - & - \\
 & HAR-RV & - & -12.6 & - & -16.9 & - & -9.6 & - & -2.3 & - & -10.3 & - & - \\
\cmidrule(lr){2-14}
 & CE & -3.4 & \textbf{+6.2} & -4.5 & \textbf{+3.9} & \textbf{-0.6} & \textbf{-3.4} & -8.3 & \textbf{-2.4} & -4.2 & \textbf{+1.1} & \textbf{-1.6} & \textbf{5/8} \\
 & MSE & -4.0 & -1.8 & -2.7 & -11.0 & -3.3 & -14.0 & -5.7 & -12.2 & -3.9 & -9.8 & -6.9 & 0/8 \\
 & \cellcolor{green!15} \textbf{VCA} & \cellcolor{green!15} \textbf{-0.6} & \cellcolor{green!15} -35.0 & \cellcolor{green!15} \textbf{+1.0} & \cellcolor{green!15} +2.6 & \cellcolor{green!15} -1.3 & \cellcolor{green!15} -4.4 & \cellcolor{green!15} \textbf{-4.9} & \cellcolor{green!15} -23.2 & \cellcolor{green!15} \textbf{-1.4} & \cellcolor{green!15} -15.0 & \cellcolor{green!15} -8.2 & \cellcolor{green!15} 3/8 \\
\midrule
\multirow{6}{*}{\rotatebox[origin=c]{90}{Crypto}} & Frozen & +0.055 & 1.2e-4 & +0.051 & 2.7e-4 & +0.051 & 5.8e-4 & +0.044 & 9.0e-4 & - & - & - & - \\
 & GARCH & - & -11.2 & - & -5.4 & - & -8.4 & - & -2.8 & - & -7.0 & - & - \\
 & HAR-RV & - & -24.8 & - & -20.8 & - & -18.8 & - & -17.2 & - & -20.4 & - & - \\
\cmidrule(lr){2-14}
 & CE & -12.0 & -1.8 & -14.3 & \textbf{+3.4} & -50.7 & \textbf{+7.4} & -11.3 & \textbf{+5.0} & -22.1 & \textbf{+3.5} & -9.3 & 3/8 \\
 & MSE & -42.1 & -9.1 & -5.7 & -10.1 & \textbf{-2.7} & -21.5 & +6.5 & -7.5 & -11.0 & -12.1 & -11.6 & 1/8 \\
 & \cellcolor{green!15} \textbf{VCA} & \cellcolor{green!15} \textbf{+3.4} & \cellcolor{green!15} \textbf{+0.6} & \cellcolor{green!15} \textbf{+7.3} & \cellcolor{green!15} -30.2 & \cellcolor{green!15} -6.1 & \cellcolor{green!15} -12.4 & \cellcolor{green!15} \textbf{+10.2} & \cellcolor{green!15} -8.7 & \cellcolor{green!15} \textbf{+3.7} & \cellcolor{green!15} -12.7 & \cellcolor{green!15} \textbf{-4.5} & \cellcolor{green!15} \textbf{4/8} \\
\bottomrule
\end{tabular}

}
\end{table}

\begin{table}[h]
\centering
\caption{Aggregate ACF$^2$ gap by horizon and dataset (lower is better). Bold marks the lowest (best) arm in each row. 
$^\dagger H{=}8$ gives only $7$ squared returns per window, short of the $K{=}15$ lags used here (needs at least $2K$ squared returns), reported for completeness.}
\label{tab:acf2}
\begin{tabular}{llrrr}
\toprule
$H$ & Dataset & CE & MSE & VCA \\
\midrule
\multirow{3}{*}{$8^\dagger$} & CSI300 & .05772{\scriptsize$\pm$.00042} & .05949{\scriptsize$\pm$.00082} & \cellcolor{green!15} \textbf{.05749{\scriptsize$\pm$.00028}} \\
 & CSI500 & .05651{\scriptsize$\pm$.00037} & .05821{\scriptsize$\pm$.00091} & \cellcolor{green!15} \textbf{.05419{\scriptsize$\pm$.00070}} \\
 & Crypto & \textbf{.06436{\scriptsize$\pm$.00007}} & .06773{\scriptsize$\pm$.00110} & \cellcolor{green!15} .06456{\scriptsize$\pm$.00006} \\
\midrule
\multirow{3}{*}{16} & CSI300 & \textbf{.03517{\scriptsize$\pm$.00015}} & .03663{\scriptsize$\pm$.00013} & \cellcolor{green!15} .03523{\scriptsize$\pm$.00053} \\
 & CSI500 & .03804{\scriptsize$\pm$.00020} & .04019{\scriptsize$\pm$.00061} & \cellcolor{green!15} \textbf{.03730{\scriptsize$\pm$.00045}} \\
 & Crypto & .03958{\scriptsize$\pm$.00008} & .04275{\scriptsize$\pm$.00091} & \cellcolor{green!15} \textbf{.03885{\scriptsize$\pm$.00026}} \\
\midrule
\multirow{3}{*}{32} & CSI300 & .01771{\scriptsize$\pm$.00004} & .01850{\scriptsize$\pm$.00014} & \cellcolor{green!15} \textbf{.01749{\scriptsize$\pm$.00003}} \\
 & CSI500 & .01773{\scriptsize$\pm$.00005} & .01843{\scriptsize$\pm$.00017} & \cellcolor{green!15} \textbf{.01755{\scriptsize$\pm$.00007}} \\
 & Crypto & \textbf{.01847{\scriptsize$\pm$.00003}} & .02076{\scriptsize$\pm$.00116} & \cellcolor{green!15} .01852{\scriptsize$\pm$.00021} \\
\midrule
\multirow{3}{*}{48} & CSI300 & .01466{\scriptsize$\pm$.00003} & .01489{\scriptsize$\pm$.00014} & \cellcolor{green!15} \textbf{.01417{\scriptsize$\pm$.00018}} \\
 & CSI500 & .01478{\scriptsize$\pm$.00008} & .01541{\scriptsize$\pm$.00001} & \cellcolor{green!15} \textbf{.01452{\scriptsize$\pm$.00036}} \\
 & Crypto & \textbf{.01692{\scriptsize$\pm$.00001}} & .01872{\scriptsize$\pm$.00056} & \cellcolor{green!15} .01754{\scriptsize$\pm$.00192} \\
\midrule
\multicolumn{2}{l}{\textbf{wins}} & 4/12 & 0/12 & \cellcolor{green!15} \textbf{8/12} \\
\multicolumn{2}{l}{\textbf{mean}} & .03264 & .03431 & \cellcolor{green!15} \textbf{.03228} \\
\bottomrule
\end{tabular}

\end{table}

\begin{table}[!htb]
\centering
\caption{Chronos backbone, $H{=}16, 32$. \textbf{Bold} indicates best score across fine-tuning methods, \underline{underlined} marks a cell where $\sigma^2$-MAE (\textsc{fore}) is orders of magnitude above Frozen; $\Delta$RankIC does not depend on that blown-up price scale. Full per-cell details in Tables~\ref{tab:chronos_raw_h16}--\ref{tab:chronos_raw_h32}.}
\label{tab:chronos_main}
\resizebox{\textwidth}{!}{
\begin{tabular}{clrrrrrrc}
\toprule
& & \multicolumn{3}{c}{$H{=}16$} & \multicolumn{3}{c}{$H{=}32$} & \\
\cmidrule(lr){3-5}\cmidrule(lr){6-8}
DS & Method & $\Delta_\text{RankIC}$ & $\Delta_\text{MAE}$ & Score & $\Delta_\text{RankIC}$ & $\Delta_\text{MAE}$ & Score & WR \\
\midrule
\multirow{4}{*}{\rotatebox[origin=c]{90}{CSI300}} & Frozen & +0.0181 & 7.4e-3 & - & +0.0042 & 1.6e-2 & - & - \\
 & CE & -70.7 & \textbf{-2.2} & -36.4 & \textbf{+364.3} & \textbf{-2.2} & \textbf{+181.1} & \textbf{4/6} \\
 & MSE & \textbf{+18.8} & -3.9 & \textbf{+7.4} & +181.0 & -2.5 & +89.2 & 2/6 \\
 & \cellcolor{green!15} \textbf{VCA} & \cellcolor{green!15} -248.4 & \cellcolor{green!15} \underline{-1649079.1} & \cellcolor{green!15} \underline{-824663.8} & \cellcolor{green!15} -512.9 & \cellcolor{green!15} \underline{-133064.2} & \cellcolor{green!15} \underline{-66788.5} & \cellcolor{green!15} 0/6 \\
\midrule
\multirow{4}{*}{\rotatebox[origin=c]{90}{CSI500}} & Frozen & +0.0537 & 9.9e-3 & - & +0.0534 & 2.2e-2 & - & - \\
 & CE & \textbf{-34.6} & \textbf{-0.6} & \textbf{-17.6} & -126.6 & \textbf{-0.7} & -63.7 & \textbf{4/6} \\
 & MSE & -235.9 & -2.9 & -119.4 & -241.8 & -2.0 & -121.9 & 0/6 \\
 & \cellcolor{green!15} \textbf{VCA} & \cellcolor{green!15} -91.2 & \cellcolor{green!15} \underline{-265545.6} & \cellcolor{green!15} \underline{-132818.4} & \cellcolor{green!15} \textbf{+138.8} & \cellcolor{green!15} -0.5 & \cellcolor{green!15} \textbf{+69.1} & \cellcolor{green!15} 2/6 \\
\midrule
\multirow{4}{*}{\rotatebox[origin=c]{90}{Crypto}} & Frozen & -0.0017 & 2.7e-4 & - & +0.0049 & 6.0e-4 & - & - \\
 & CE & \textbf{+970.6} & \textbf{-5.9} & \textbf{+482.3} & \textbf{+230.6} & -6.5 & \textbf{+112.0} & \textbf{5/6} \\
 & MSE & +747.1 & -8.9 & +369.1 & -173.5 & \textbf{-6.2} & -89.9 & 1/6 \\
 & \cellcolor{green!15} \textbf{VCA} & \cellcolor{green!15} +370.6 & \cellcolor{green!15} -6.7 & \cellcolor{green!15} +181.9 & \cellcolor{green!15} -77.5 & \cellcolor{green!15} \underline{-2512388.9} & \cellcolor{green!15} \underline{-1256233.2} & \cellcolor{green!15} 0/6 \\
\bottomrule
\end{tabular}

}
\end{table}


\begin{table}[h]
\centering
\caption{Same regime-match exercise as Table~\ref{tab:matchrule}, with the raw lag-1 sample ACF of
squared returns $\hat\rho_1(\cdot)$ in place of $\mathrm{vol}(\cdot)$. Bold marks a cell where $\hat\rho_1$ disagrees with
$\mathrm{vol}(\cdot)$'s match in Table~\ref{tab:matchrule}.}
\label{tab:matchrule_acf1}
\begin{tabular}{llccccc}
\toprule
Dataset & Slice & $\hat\rho_1$(first) & $\hat\rho_1$(last) & $\hat\rho_1$(test) & $\hat\rho_1$(val) & test-closer / val-closer \\
\midrule
\multirow{2}{*}{CSI300} & 10\% & +0.099 & +0.102 & +0.041 & +0.054 & first / first \\
 & 1\% & +0.070 & +0.058 & +0.041 & +0.054 & \textbf{last} / \textbf{last} \\
\midrule
\multirow{2}{*}{CSI500} & 10\% & +0.172 & +0.127 & +0.047 & +0.050 & \textbf{last} / \textbf{last} \\
 & 1\% & +0.148 & +0.074 & +0.047 & +0.050 & last / last \\
\midrule
\multirow{2}{*}{Crypto} & 10\% & +0.306 & +0.134 & +0.207 & +0.210 & last / last \\
 & 1\% & +0.247 & +0.179 & +0.207 & +0.210 & last / last \\
\bottomrule
\end{tabular}

\end{table}

\begin{table}[h]
\centering
\caption{Full per-cell numbers behind Figure~\ref{fig:dataeff}, under both evaluation conventions. Bold marks the best
method in each column within a (convention, dataset). Shaded cells indicate \emph{matched regime}.}
\label{tab:dataeff}
\begin{tabular}{cllccccc}
\toprule
& Dataset & Method & $1\%$-first & $1\%$-last & $10\%$-first & $10\%$-last & $100\%$ \\
\midrule
\multirow{9}{*}{\rotatebox{90}{RankIC}}
& CSI300 & CE & \cellcolor{green!15} \textbf{+1.6{\scriptsize$\pm$0.2}} & \textbf{-18.3{\scriptsize$\pm$0.4}} & \cellcolor{green!15} \textbf{-0.1{\scriptsize$\pm$1.1}} & \textbf{+0.7{\scriptsize$\pm$1.1}} & -11.8{\scriptsize$\pm$1.1} \\
& & MSE & \cellcolor{green!15} -16.5{\scriptsize$\pm$1.7} & -32.9{\scriptsize$\pm$4.5} & \cellcolor{green!15} -22.0{\scriptsize$\pm$6.4} & -34.7{\scriptsize$\pm$4.1} & -6.5{\scriptsize$\pm$1.5} \\
& & VCA & \cellcolor{green!15} -1.3{\scriptsize$\pm$1.3} & -23.0{\scriptsize$\pm$4.2} & \cellcolor{green!15} -0.4{\scriptsize$\pm$12.4} & -28.7{\scriptsize$\pm$9.7} & \textbf{+8.0{\scriptsize$\pm$1.0}} \\
\cmidrule(lr){2-8}
& CSI500 & CE & \cellcolor{green!15} \textbf{+2.0{\scriptsize$\pm$0.1}} & -5.9{\scriptsize$\pm$0.1} & \cellcolor{green!15} -0.1{\scriptsize$\pm$0.1} & -3.6{\scriptsize$\pm$0.1} & \textbf{-0.6{\scriptsize$\pm$0.1}} \\
& & MSE & \cellcolor{green!15} -5.1{\scriptsize$\pm$0.6} & -10.2{\scriptsize$\pm$0.1} & \cellcolor{green!15} -6.5{\scriptsize$\pm$0.8} & -10.0{\scriptsize$\pm$0.2} & -3.3{\scriptsize$\pm$0.8} \\
& & VCA & \cellcolor{green!15} +1.0{\scriptsize$\pm$1.0} & \textbf{-0.5{\scriptsize$\pm$2.6}} & \cellcolor{green!15} \textbf{+0.2{\scriptsize$\pm$1.1}} & \textbf{-1.8{\scriptsize$\pm$3.0}} & -1.3{\scriptsize$\pm$2.1} \\
\cmidrule(lr){2-8}
& Crypto & CE & \textbf{-10.7{\scriptsize$\pm$4.0}} & \cellcolor{green!15} -1.0{\scriptsize$\pm$1.6} & \textbf{-8.7{\scriptsize$\pm$9.0}} & \cellcolor{green!15} -16.7{\scriptsize$\pm$2.9} & -50.7{\scriptsize$\pm$3.5} \\
& & MSE & -41.1{\scriptsize$\pm$8.4} & \cellcolor{green!15} -23.8{\scriptsize$\pm$10.7} & -25.8{\scriptsize$\pm$4.1} & \cellcolor{green!15} -17.3{\scriptsize$\pm$14.5} & \textbf{-2.7{\scriptsize$\pm$12.0}} \\
& & VCA & -15.2{\scriptsize$\pm$19.0} & \cellcolor{green!15} \textbf{+12.2{\scriptsize$\pm$13.5}} & -33.5{\scriptsize$\pm$4.6} & \cellcolor{green!15} \textbf{+7.0{\scriptsize$\pm$12.9}} & -6.1{\scriptsize$\pm$19.0} \\
\midrule
\multirow{9}{*}{\rotatebox{90}{\textsc{fore}}}
& CSI300 & CE & \cellcolor{green!15} -0.5{\scriptsize$\pm$0.1} & -10.2{\scriptsize$\pm$0.2} & \cellcolor{green!15} -0.8{\scriptsize$\pm$0.5} & \textbf{-4.3{\scriptsize$\pm$0.5}} & -11.8{\scriptsize$\pm$0.5} \\
& & MSE & \cellcolor{green!15} -15.5{\scriptsize$\pm$1.0} & -26.8{\scriptsize$\pm$2.5} & \cellcolor{green!15} -18.1{\scriptsize$\pm$2.3} & -27.2{\scriptsize$\pm$1.7} & -9.2{\scriptsize$\pm$0.9} \\
& & VCA & \cellcolor{green!15} \textbf{+1.3{\scriptsize$\pm$5.8}} & \textbf{-14.6{\scriptsize$\pm$2.1}} & \cellcolor{green!15} \textbf{+3.4{\scriptsize$\pm$3.8}} & -15.1{\scriptsize$\pm$4.9} & \textbf{+1.6{\scriptsize$\pm$1.3}} \\
\cmidrule(lr){2-8}
& CSI500 & CE & \cellcolor{green!15} +0.8{\scriptsize$\pm$0.0} & -6.0{\scriptsize$\pm$0.0} & \cellcolor{green!15} +1.9{\scriptsize$\pm$0.0} & -5.0{\scriptsize$\pm$0.1} & -3.3{\scriptsize$\pm$0.1} \\
& & MSE & \cellcolor{green!15} -8.7{\scriptsize$\pm$0.6} & -13.1{\scriptsize$\pm$0.4} & \cellcolor{green!15} -7.6{\scriptsize$\pm$1.1} & -13.6{\scriptsize$\pm$0.3} & -4.4{\scriptsize$\pm$1.2} \\
& & VCA & \cellcolor{green!15} \textbf{+1.3{\scriptsize$\pm$4.1}} & \textbf{-0.9{\scriptsize$\pm$2.3}} & \cellcolor{green!15} \textbf{+5.0{\scriptsize$\pm$1.7}} & \textbf{+1.5{\scriptsize$\pm$1.9}} & \textbf{+1.3{\scriptsize$\pm$0.7}} \\
\cmidrule(lr){2-8}
& Crypto & CE & \textbf{-2.7{\scriptsize$\pm$1.9}} & \cellcolor{green!15} +1.3{\scriptsize$\pm$0.8} & \textbf{-2.7{\scriptsize$\pm$4.5}} & \cellcolor{green!15} -6.4{\scriptsize$\pm$1.5} & -22.3{\scriptsize$\pm$1.7} \\
& & MSE & -22.3{\scriptsize$\pm$4.1} & \cellcolor{green!15} -10.7{\scriptsize$\pm$5.0} & -13.9{\scriptsize$\pm$2.2} & \cellcolor{green!15} -7.5{\scriptsize$\pm$7.1} & \textbf{-4.7{\scriptsize$\pm$6.6}} \\
& & VCA & -5.8{\scriptsize$\pm$9.6} & \cellcolor{green!15} \textbf{+8.7{\scriptsize$\pm$4.9}} & -13.8{\scriptsize$\pm$2.7} & \cellcolor{green!15} \textbf{+4.4{\scriptsize$\pm$6.3}} & -5.5{\scriptsize$\pm$11.1} \\
\midrule
\multirow{9}{*}{\rotatebox{90}{\textsc{vol}}}
& CSI300 & CE & \cellcolor{green!15} \textbf{+1.0{\scriptsize$\pm$0.1}} & \textbf{-8.6{\scriptsize$\pm$0.2}} & \cellcolor{green!15} -1.6{\scriptsize$\pm$1.2} & \textbf{-1.8{\scriptsize$\pm$0.6}} & -10.0{\scriptsize$\pm$0.5} \\
& & MSE & \cellcolor{green!15} -18.6{\scriptsize$\pm$1.0} & -30.4{\scriptsize$\pm$2.9} & \cellcolor{green!15} -21.7{\scriptsize$\pm$1.8} & -30.0{\scriptsize$\pm$1.4} & -13.1{\scriptsize$\pm$0.9} \\
& & VCA & \cellcolor{green!15} -1.3{\scriptsize$\pm$0.9} & -12.8{\scriptsize$\pm$1.8} & \cellcolor{green!15} \textbf{+2.5{\scriptsize$\pm$5.6}} & -13.6{\scriptsize$\pm$4.7} & \textbf{+2.7{\scriptsize$\pm$4.0}} \\
\cmidrule(lr){2-8}
& CSI500 & CE & \cellcolor{green!15} \textbf{+1.5{\scriptsize$\pm$0.2}} & -4.1{\scriptsize$\pm$0.0} & \cellcolor{green!15} -0.2{\scriptsize$\pm$0.1} & -2.4{\scriptsize$\pm$0.2} & \textbf{-2.0{\scriptsize$\pm$0.2}} \\
& & MSE & \cellcolor{green!15} -12.5{\scriptsize$\pm$0.8} & -16.7{\scriptsize$\pm$0.7} & \cellcolor{green!15} -11.5{\scriptsize$\pm$1.6} & -17.4{\scriptsize$\pm$0.7} & -8.7{\scriptsize$\pm$1.3} \\
& & VCA & \cellcolor{green!15} +0.1{\scriptsize$\pm$2.1} & \textbf{-1.2{\scriptsize$\pm$2.8}} & \cellcolor{green!15} \textbf{+3.1{\scriptsize$\pm$1.0}} & \textbf{+3.4{\scriptsize$\pm$3.1}} & -2.9{\scriptsize$\pm$4.1} \\
\cmidrule(lr){2-8}
& Crypto & CE & \textbf{-1.4{\scriptsize$\pm$1.6}} & \cellcolor{green!15} +2.9{\scriptsize$\pm$0.9} & \textbf{-2.4{\scriptsize$\pm$4.5}} & \cellcolor{green!15} -5.8{\scriptsize$\pm$1.5} & -21.7{\scriptsize$\pm$1.7} \\
& & MSE & -25.9{\scriptsize$\pm$3.7} & \cellcolor{green!15} -12.0{\scriptsize$\pm$5.0} & -17.9{\scriptsize$\pm$2.1} & \cellcolor{green!15} -9.4{\scriptsize$\pm$7.3} & -12.1{\scriptsize$\pm$7.3} \\
& & VCA & -4.7{\scriptsize$\pm$9.4} & \cellcolor{green!15} \textbf{+6.4{\scriptsize$\pm$3.4}} & -21.6{\scriptsize$\pm$10.6} & \cellcolor{green!15} +1.7{\scriptsize$\pm$6.6} & \textbf{-9.3{\scriptsize$\pm$12.8}} \\
\bottomrule
\end{tabular}

\end{table}

\begin{figure}[h]
\centering
\includegraphics[width=\textwidth]{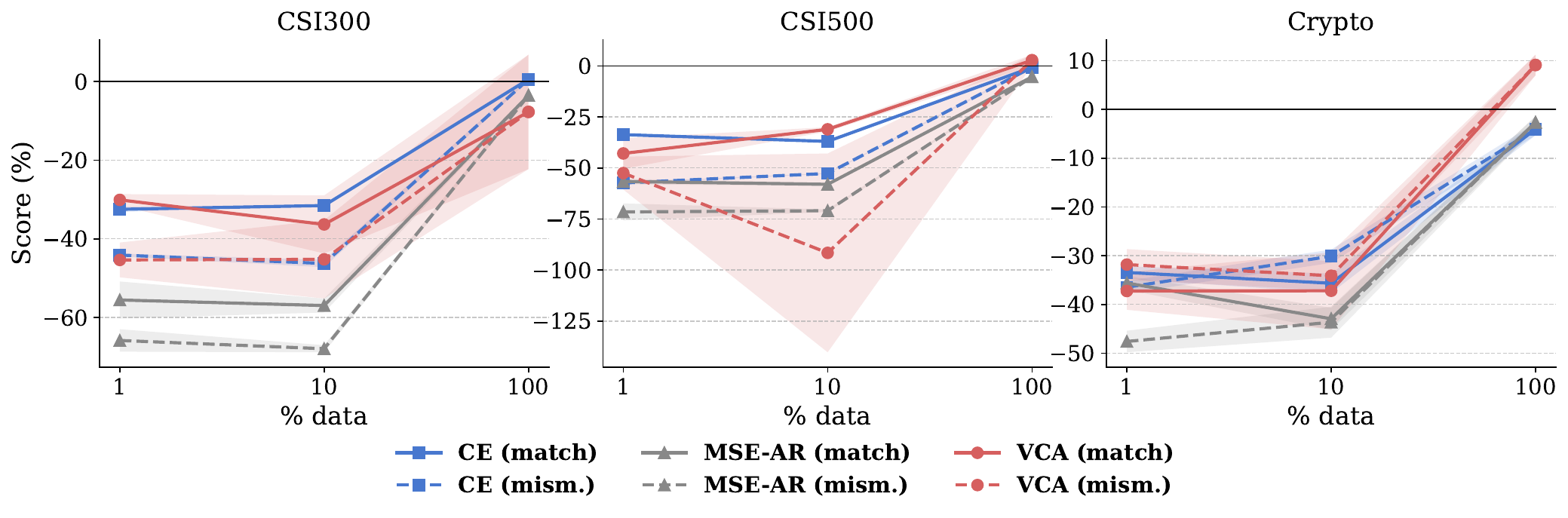}
\caption{Same layout as Figure~\ref{fig:dataeff}, at $H{=}16$: every fine-tuning method stays well below the frozen
model at $1\%$ and $10\%$ data, matched or mismatched, recovering only at $100\%$.}
\label{fig:dataeff16}
\end{figure}

\begin{table}[h]
\centering
\caption{Per-cell values for ablation at $H{=}16$, under \textsc{fore}, mean over three seeds. \emph{No setting is uniformly optimal across all three datasets}.}
\label{tab:abl}
\resizebox{\textwidth}{!}{
\begin{tabular}{llcccccc}
\toprule
Ablation & Value & \multicolumn{2}{c}{CSI300} & \multicolumn{2}{c}{CSI500} & \multicolumn{2}{c}{Crypto} \\
\cmidrule(lr){3-4}\cmidrule(lr){5-6}\cmidrule(lr){7-8}
 & & RankIC & $\sigma^2$-MAE & RankIC & $\sigma^2$-MAE & RankIC & $\sigma^2$-MAE \\
\midrule
\multirow[c]{5}{*}{Lags $K$} & \textbf{1} & \cellcolor{green!15} +0.3280 & \cellcolor{green!15} 0.00788 & \cellcolor{green!15} +0.8615 & \cellcolor{green!15} 0.00860 & \cellcolor{green!15} +0.0552 & \cellcolor{green!15} 0.00028 \\
 & 2 & +0.0480 & 0.00614 & +0.1242 & 0.00834 & +0.0147 & 0.00030 \\
 & 4 & +0.0447 & 0.00626 & +0.1437 & 0.00830 & +0.0106 & 0.00289 \\
 & 8 & +0.0818 & 0.00626 & +0.1582 & 0.00859 & +0.0199 & 0.00030 \\
\midrule
\multirow[c]{3}{*}{Lookback $W$} & 80 & +0.0902 & 0.00652 & +0.0774 & 0.00824 & +0.0072 & 0.00031 \\
 & \textbf{160} & \cellcolor{green!15} +0.3280 & \cellcolor{green!15} 0.00754 & \cellcolor{green!15} +0.8615 & \cellcolor{green!15} 0.00852 & \cellcolor{green!15} +0.0552 & \cellcolor{green!15} 0.00028 \\
 & 320 & +0.0773 & 0.01192 & +0.1930 & 0.00849 & +0.0073 & 0.00029 \\
\midrule
\multirow[c]{4}{*}{Weight $\lambda$} & 5 & +0.1312 & 0.00669 & +0.1178 & 0.00874 & +0.0092 & 0.00040 \\
 & \textbf{10} & \cellcolor{green!15} +0.3280 & \cellcolor{green!15} 0.00788 & \cellcolor{green!15} +0.8615 & \cellcolor{green!15} 0.00860 & \cellcolor{green!15} +0.0552 & \cellcolor{green!15} 0.00028 \\
 & 20 & +0.0557 & 0.00652 & +0.0405 & 0.00852 & +0.0176 & 0.00030 \\
 & 40 & +0.0478 & 0.00631 & +0.1343 & 0.00823 & +0.0191 & 0.00035 \\
\bottomrule
\end{tabular}

}
\end{table}

\begin{table}[h]
\centering
\caption{Same ablations as Table~\ref{tab:abl}, at $H{=}32$.}
\label{tab:abl_h32}
\resizebox{\textwidth}{!}{
\begin{tabular}{llcccccc}
\toprule
Ablation & Value & \multicolumn{2}{c}{CSI300} & \multicolumn{2}{c}{CSI500} & \multicolumn{2}{c}{Crypto} \\
\cmidrule(lr){3-4}\cmidrule(lr){5-6}\cmidrule(lr){7-8}
 & & RankIC & $\sigma^2$-MAE & RankIC & $\sigma^2$-MAE & RankIC & $\sigma^2$-MAE \\
\midrule
\multirow[c]{5}{*}{Lags $K$} & \textbf{1} & \cellcolor{green!15} +0.3596 & \cellcolor{green!15} 0.01501 & \cellcolor{green!15} +0.8522 & \cellcolor{green!15} 0.01882 & \cellcolor{green!15} +0.0482 & \cellcolor{green!15} 0.00071 \\
 & 2 & +0.3162 & 0.01494 & +0.8280 & 0.01955 & +0.0445 & 0.00066 \\
 & 4 & +0.2739 & 0.01441 & +0.8473 & 0.01813 & +0.0491 & 0.00065 \\
 & 8 & +0.2931 & 0.01414 & +0.8683 & 0.02492 & +0.0569 & 0.00069 \\
\midrule
\multirow[c]{3}{*}{Lookback $W$} & 80 & +0.3105 & 0.01427 & +0.8272 & 0.02659 & +0.0413 & 0.00058 \\
 & \textbf{160} & \cellcolor{green!15} +0.3596 & \cellcolor{green!15} 0.01454 & \cellcolor{green!15} +0.8522 & \cellcolor{green!15} 0.01858 & \cellcolor{green!15} +0.0482 & \cellcolor{green!15} 0.00071 \\
 & 320 & +0.2761 & 0.01604 & +0.8607 & 0.01942 & +0.0493 & 0.00070 \\
\midrule
\multirow[c]{4}{*}{Weight $\lambda$} & 5 & +0.3435 & 0.01489 & +0.8345 & 0.01784 & +0.0450 & 0.00065 \\
 & \textbf{10} & \cellcolor{green!15} +0.3596 & \cellcolor{green!15} 0.01501 & \cellcolor{green!15} +0.8522 & \cellcolor{green!15} 0.01882 & \cellcolor{green!15} +0.0482 & \cellcolor{green!15} 0.00071 \\
 & 20 & +0.2632 & 0.01378 & +0.8719 & 0.01877 & +0.0569 & 0.00065 \\
 & 40 & +0.0520 & 0.01364 & +0.1267 & 0.02624 & +0.0189 & 0.00066 \\
\bottomrule
\end{tabular}

}
\end{table}

\begin{table}
\centering
\caption{Relative weight drift $\|\theta-\theta_0\|_2/\|\theta_0\|_2$ from the frozen model, whole-model
(\emph{all}) weights, in units of $10^{-2}$, mean over three seeds. Ratio is VCA/CE. The $10\%$ rows (shaded)
are the regime-matched slice.}
\label{tab:drift}
\begin{tabular}{ll rrr rrr}
\toprule
& & \multicolumn{3}{c}{$H{=}16$} & \multicolumn{3}{c}{$H{=}32$} \\
\cmidrule(lr){3-5}\cmidrule(lr){6-8}
\%data & Dataset & CE & VCA & Ratio & CE & VCA & Ratio \\
\midrule
\multirow{3}{*}{100\%} & CSI300 & 2.34 & 3.38 & 1.4$\times$ & 0.90 & 3.78 & 4.2$\times$ \\
 & CSI500 & 4.70 & 5.77 & 1.2$\times$ & 1.38 & 4.95 & 3.6$\times$ \\
 & Crypto & 5.08 & 6.95 & 1.4$\times$ & 1.94 & 5.77 & 3.0$\times$ \\
\midrule
\multirow{3}{*}{10\%} & CSI300 & \cellcolor{green!15} 0.49 & \cellcolor{green!15} 0.85 & \cellcolor{green!15} 1.7$\times$ & \cellcolor{green!15} 0.60 & \cellcolor{green!15} 0.81 & \cellcolor{green!15} 1.3$\times$ \\
 & CSI500 & \cellcolor{green!15} 1.18 & \cellcolor{green!15} 1.29 & \cellcolor{green!15} 1.1$\times$ & \cellcolor{green!15} 0.96 & \cellcolor{green!15} 1.08 & \cellcolor{green!15} 1.1$\times$ \\
 & Crypto & \cellcolor{green!15} 1.57 & \cellcolor{green!15} 1.68 & \cellcolor{green!15} 1.1$\times$ & \cellcolor{green!15} 1.39 & \cellcolor{green!15} 1.43 & \cellcolor{green!15} 1.0$\times$ \\
\midrule
\multirow{3}{*}{1\%} & CSI300 & 0.54 & 0.42 & 0.8$\times$ & 0.25 & 0.45 & 1.8$\times$ \\
 & CSI500 & 0.71 & 0.38 & 0.5$\times$ & 0.42 & 0.50 & 1.2$\times$ \\
 & Crypto & 0.37 & 0.55 & 1.5$\times$ & 0.35 & 0.39 & 1.1$\times$ \\
\bottomrule
\end{tabular}

\end{table}

\begin{table}[p]
\centering
\scriptsize
\caption{Raw per-cell results at $H{=}8$, mean$\pm$sd over 3 seeds. No quantity is normalized against the frozen model. $\sigma^2$-MAE columns are reported to 5 decimal places; all other columns to 4.}
\label{tab:raw8}
\setlength{\tabcolsep}{3pt}
\resizebox{\textwidth}{!}{%
\begin{tabular}{llccccccc}
\toprule
DS & Arm & price RankIC & price IC & ret.\ RankIC & ret.\ IC & ACF$^2$ gap & $\sigma^2$-MAE\textsubscript{\textsc{fore}} & $\sigma^2$-MAE\textsubscript{\textsc{vol}} \\
\midrule
\multirow{4}{*}{\rotatebox[origin=c]{90}{CSI300}} & Frozen & +0.2906 & +0.0630 & +0.0537 & +0.0142 & 0.05812 & 0.00345 & 0.00342 \\
 & CE & +0.2539 & +0.0044 & -0.0092 & -0.0100 & 0.05772 & 0.00322 & 0.00313 \\
 & MSE & +0.2418 & -0.0096 & +0.0031 & -0.0025 & 0.05949 & 0.00363 & 0.00349 \\
 & \cellcolor{green!15} \textbf{VCA} & \cellcolor{green!15} +0.2731 & \cellcolor{green!15} +0.0331 & \cellcolor{green!15} +0.0115 & \cellcolor{green!15} -0.0067 & \cellcolor{green!15} 0.05749 & \cellcolor{green!15} 0.00310 & \cellcolor{green!15} 0.00312 \\
\midrule
\multirow{4}{*}{\rotatebox[origin=c]{90}{CSI500}} & Frozen & +0.8485 & +0.0985 & +0.0411 & +0.0444 & 0.05537 & 0.00446 & 0.00448 \\
 & CE & +0.8196 & -0.0787 & +0.0168 & +0.0185 & 0.05651 & 0.00419 & 0.00420 \\
 & MSE & +0.8148 & -0.1105 & -0.0416 & -0.0551 & 0.05821 & 0.00469 & 0.00456 \\
 & \cellcolor{green!15} \textbf{VCA} & \cellcolor{green!15} +0.8437 & \cellcolor{green!15} +0.0569 & \cellcolor{green!15} +0.0122 & \cellcolor{green!15} +0.0030 & \cellcolor{green!15} 0.05419 & \cellcolor{green!15} 0.00472 & \cellcolor{green!15} 0.00605 \\
\midrule
\multirow{4}{*}{\rotatebox[origin=c]{90}{Crypto}} & Frozen & +0.0546 & +0.0178 & +0.0127 & +0.0104 & 0.06489 & 0.00014 & 0.00012 \\
 & CE & +0.0480 & +0.0099 & +0.0089 & +0.0127 & 0.06436 & 0.00013 & 0.00013 \\
 & MSE & +0.0316 & +0.0057 & +0.0120 & +0.0140 & 0.06773 & 0.00014 & 0.00014 \\
 & \cellcolor{green!15} \textbf{VCA} & \cellcolor{green!15} +0.0564 & \cellcolor{green!15} +0.0195 & \cellcolor{green!15} +0.0148 & \cellcolor{green!15} +0.0130 & \cellcolor{green!15} 0.06456 & \cellcolor{green!15} 0.00013 & \cellcolor{green!15} 0.00012 \\
\bottomrule
\end{tabular}

}
\end{table}

\begin{table}[p]
\centering
\scriptsize
\caption{Raw per-cell results at $H{=}16$, mean$\pm$sd over 3 seeds. No quantity is normalized against the frozen model. $\sigma^2$-MAE columns are reported to 5 decimal places; all other columns to 4.}
\label{tab:raw16}
\setlength{\tabcolsep}{3pt}
\resizebox{\textwidth}{!}{%
\begin{tabular}{llccccccc}
\toprule
DS & Arm & price RankIC & price IC & ret.\ RankIC & ret.\ IC & ACF$^2$ gap & $\sigma^2$-MAE\textsubscript{\textsc{fore}} & $\sigma^2$-MAE\textsubscript{\textsc{vol}} \\
\midrule
\multirow{4}{*}{\rotatebox[origin=c]{90}{CSI300}} & Frozen & +0.3167 & +0.0817 & +0.0240 & -0.0304 & 0.03641 & 0.00662 & 0.00612 \\
 & CE & +0.3405 & +0.1118 & +0.0304 & +0.0183 & 0.03517 & 0.00704 & 0.00618 \\
 & MSE & +0.3352 & +0.1045 & +0.0424 & +0.0463 & 0.03663 & 0.00747 & 0.00708 \\
 & \cellcolor{green!15} \textbf{VCA} & \cellcolor{green!15} +0.3280 & \cellcolor{green!15} +0.0990 & \cellcolor{green!15} +0.0219 & \cellcolor{green!15} +0.0030 & \cellcolor{green!15} 0.03523 & \cellcolor{green!15} 0.00788 & \cellcolor{green!15} 0.00812 \\
\midrule
\multirow{4}{*}{\rotatebox[origin=c]{90}{CSI500}} & Frozen & +0.8529 & +0.1210 & +0.0642 & +0.0314 & 0.03759 & 0.00901 & 0.00840 \\
 & CE & +0.8146 & -0.1335 & +0.0138 & +0.0373 & 0.03804 & 0.00873 & 0.00807 \\
 & MSE & +0.8300 & -0.0373 & +0.0219 & +0.0688 & 0.04019 & 0.00973 & 0.00932 \\
 & \cellcolor{green!15} \textbf{VCA} & \cellcolor{green!15} +0.8615 & \cellcolor{green!15} +0.1723 & \cellcolor{green!15} +0.0480 & \cellcolor{green!15} +0.0438 & \cellcolor{green!15} 0.03730 & \cellcolor{green!15} 0.00860 & \cellcolor{green!15} 0.00819 \\
\midrule
\multirow{4}{*}{\rotatebox[origin=c]{90}{Crypto}} & Frozen & +0.0514 & +0.0139 & +0.0147 & +0.0053 & 0.04047 & 0.00031 & 0.00027 \\
 & CE & +0.0441 & +0.0042 & +0.0078 & +0.0136 & 0.03958 & 0.00029 & 0.00026 \\
 & MSE & +0.0485 & +0.0079 & +0.0083 & +0.0120 & 0.04275 & 0.00031 & 0.00030 \\
 & \cellcolor{green!15} \textbf{VCA} & \cellcolor{green!15} +0.0552 & \cellcolor{green!15} +0.0167 & \cellcolor{green!15} +0.0143 & \cellcolor{green!15} +0.0067 & \cellcolor{green!15} 0.03885 & \cellcolor{green!15} 0.00028 & \cellcolor{green!15} 0.00035 \\
\bottomrule
\end{tabular}

}
\end{table}

\begin{table}[p]
\centering
\scriptsize
\caption{Raw per-cell results at $H{=}32$, mean$\pm$sd over 3 seeds. No quantity is normalized against the frozen model. $\sigma^2$-MAE columns are reported to 5 decimal places; all other columns to 4.}
\label{tab:raw32}
\setlength{\tabcolsep}{3pt}
\resizebox{\textwidth}{!}{%
\begin{tabular}{llccccccc}
\toprule
DS & Arm & price RankIC & price IC & ret.\ RankIC & ret.\ IC & ACF$^2$ gap & $\sigma^2$-MAE\textsubscript{\textsc{fore}} & $\sigma^2$-MAE\textsubscript{\textsc{vol}} \\
\midrule
\multirow{4}{*}{\rotatebox[origin=c]{90}{CSI300}} & Frozen & +0.3328 & +0.1078 & +0.0164 & -0.0309 & 0.01848 & 0.01430 & 0.01260 \\
 & CE & +0.2936 & +0.0496 & +0.0380 & +0.0517 & 0.01771 & 0.01600 & 0.01360 \\
 & MSE & +0.3111 & +0.0759 & +0.0163 & +0.0262 & 0.01850 & 0.01600 & 0.01510 \\
 & \cellcolor{green!15} \textbf{VCA} & \cellcolor{green!15} +0.3596 & \cellcolor{green!15} +0.1472 & \cellcolor{green!15} +0.0595 & \cellcolor{green!15} +0.0507 & \cellcolor{green!15} 0.01749 & \cellcolor{green!15} 0.01500 & \cellcolor{green!15} 0.01290 \\
\midrule
\multirow{4}{*}{\rotatebox[origin=c]{90}{CSI500}} & Frozen & +0.8638 & +0.1946 & +0.0933 & +0.0321 & 0.01815 & 0.01960 & 0.01710 \\
 & CE & +0.8583 & +0.1606 & +0.0029 & -0.0089 & 0.01773 & 0.02080 & 0.01770 \\
 & MSE & +0.8349 & -0.0004 & -0.0273 & +0.0147 & 0.01843 & 0.02060 & 0.01950 \\
 & \cellcolor{green!15} \textbf{VCA} & \cellcolor{green!15} +0.8522 & \cellcolor{green!15} +0.1285 & \cellcolor{green!15} +0.0132 & \cellcolor{green!15} -0.0120 & \cellcolor{green!15} 0.01755 & \cellcolor{green!15} 0.01880 & \cellcolor{green!15} 0.01790 \\
\midrule
\multirow{4}{*}{\rotatebox[origin=c]{90}{Crypto}} & Frozen & +0.0513 & +0.0136 & +0.0031 & -0.0035 & 0.01965 & 0.00068 & 0.00058 \\
 & CE & +0.0253 & -0.0125 & -0.0125 & -0.0100 & 0.01847 & 0.00064 & 0.00054 \\
 & MSE & +0.0499 & +0.0108 & +0.0100 & +0.0018 & 0.02076 & 0.00073 & 0.00070 \\
 & \cellcolor{green!15} \textbf{VCA} & \cellcolor{green!15} +0.0482 & \cellcolor{green!15} +0.0091 & \cellcolor{green!15} -0.0024 & \cellcolor{green!15} -0.0031 & \cellcolor{green!15} 0.01852 & \cellcolor{green!15} 0.00072 & \cellcolor{green!15} 0.00065 \\
\bottomrule
\end{tabular}

}
\end{table}

\begin{table}[p]
\centering
\scriptsize
\caption{Raw per-cell results at $H{=}48$, mean$\pm$sd over 3 seeds. No quantity is normalized against the frozen model. $\sigma^2$-MAE columns are reported to 5 decimal places; all other columns to 4.}
\label{tab:raw48}
\setlength{\tabcolsep}{3pt}
\resizebox{\textwidth}{!}{%
\begin{tabular}{llccccccc}
\toprule
DS & Arm & price RankIC & price IC & ret.\ RankIC & ret.\ IC & ACF$^2$ gap & $\sigma^2$-MAE\textsubscript{\textsc{fore}} & $\sigma^2$-MAE\textsubscript{\textsc{vol}} \\
\midrule
\multirow{4}{*}{\rotatebox[origin=c]{90}{CSI300}} & Frozen & +0.3329 & +0.1110 & -0.0155 & -0.0548 & 0.01618 & 0.02290 & 0.01970 \\
 & CE & +0.3035 & +0.0656 & +0.0150 & +0.0129 & 0.01466 & 0.02490 & 0.02160 \\
 & MSE & +0.3274 & +0.1065 & +0.0211 & +0.0328 & 0.01489 & 0.02400 & 0.02220 \\
 & \cellcolor{green!15} \textbf{VCA} & \cellcolor{green!15} +0.3122 & \cellcolor{green!15} +0.0842 & \cellcolor{green!15} +0.0246 & \cellcolor{green!15} +0.0182 & \cellcolor{green!15} 0.01417 & \cellcolor{green!15} 0.02130 & \cellcolor{green!15} 0.02020 \\
\midrule
\multirow{4}{*}{\rotatebox[origin=c]{90}{CSI500}} & Frozen & +0.8800 & +0.2900 & +0.1163 & +0.0429 & 0.01583 & 0.03110 & 0.02680 \\
 & CE & +0.8068 & -0.0934 & +0.0007 & -0.0420 & 0.01478 & 0.03150 & 0.02740 \\
 & MSE & +0.8300 & +0.0201 & +0.0198 & +0.0322 & 0.01541 & 0.03180 & 0.03010 \\
 & \cellcolor{green!15} \textbf{VCA} & \cellcolor{green!15} +0.8370 & \cellcolor{green!15} +0.0666 & \cellcolor{green!15} +0.0461 & \cellcolor{green!15} +0.0532 & \cellcolor{green!15} 0.01452 & \cellcolor{green!15} 0.03010 & \cellcolor{green!15} 0.03300 \\
\midrule
\multirow{4}{*}{\rotatebox[origin=c]{90}{Crypto}} & Frozen & +0.0440 & +0.0008 & -0.0088 & -0.0159 & 0.01938 & 0.00105 & 0.00089 \\
 & CE & +0.0390 & -0.0005 & -0.0053 & -0.0029 & 0.01692 & 0.00101 & 0.00085 \\
 & MSE & +0.0468 & +0.0051 & +0.0095 & +0.0019 & 0.01872 & 0.00102 & 0.00096 \\
 & \cellcolor{green!15} \textbf{VCA} & \cellcolor{green!15} +0.0485 & \cellcolor{green!15} +0.0078 & \cellcolor{green!15} -0.0049 & \cellcolor{green!15} -0.0042 & \cellcolor{green!15} 0.01754 & \cellcolor{green!15} 0.00099 & \cellcolor{green!15} 0.00097 \\
\bottomrule
\end{tabular}

}
\end{table}


\begin{table}[!htb]
\centering
\small
\caption{Raw per-cell results at $H{=}16$ (Chronos). $\sigma^2$-MAE columns are reported to 5 decimal places, 1 for the Chronos-VCA rows; all other columns to 4.}
\label{tab:chronos_raw_h16}
\resizebox{\textwidth}{!}{%
\begin{tabular}{llccccccc}
\toprule
DS & Method & price RankIC & price IC & ret.\ RankIC & ret.\ IC & ACF$^2$ gap & $\sigma^2$-MAE\textsubscript{\textsc{fore}} & $\sigma^2$-MAE\textsubscript{\textsc{vol}} \\
\midrule
\multirow{4}{*}{\rotatebox[origin=c]{90}{CSI300}} & Frozen & +0.0181 & +0.0189 & +0.0221 & +0.0159 & 0.03561 & 0.00745 & 0.00650 \\
 & CE & +0.0053 & +0.0055 & +0.0498 & +0.0397 & 0.03556 & 0.00761 & 0.00670 \\
 & MSE & +0.0215 & +0.0220 & +0.0654 & +0.0610 & 0.03551 & 0.00774 & 0.00703 \\
 & \cellcolor{green!15} \textbf{VCA} & \cellcolor{green!15} -0.0269 & \cellcolor{green!15} -0.0265 & \cellcolor{green!15} -0.0059 & \cellcolor{green!15} +0.0092 & \cellcolor{green!15} 0.03369 & \cellcolor{green!15} 123.0 & \cellcolor{green!15} 189.0 \\
\midrule
\multirow{4}{*}{\rotatebox[origin=c]{90}{CSI500}} & Frozen & +0.0537 & +0.0589 & +0.0451 & +0.0206 & 0.03777 & 0.00994 & 0.00868 \\
 & CE & +0.0351 & +0.0425 & +0.0740 & +0.0601 & 0.03736 & 0.01000 & 0.00883 \\
 & MSE & -0.0730 & -0.0698 & +0.0763 & +0.0504 & 0.03709 & 0.01020 & 0.00932 \\
 & \cellcolor{green!15} \textbf{VCA} & \cellcolor{green!15} +0.0047 & \cellcolor{green!15} +0.0193 & \cellcolor{green!15} +0.0474 & \cellcolor{green!15} +0.0488 & \cellcolor{green!15} 0.03691 & \cellcolor{green!15} 26.4 & \cellcolor{green!15} 47.2 \\
\midrule
\multirow{4}{*}{\rotatebox[origin=c]{90}{Crypto}} & Frozen & -0.0017 & -0.0023 & -0.0145 & -0.0165 & 0.03828 & 0.00027 & 0.00025 \\
 & CE & +0.0148 & +0.0155 & +0.0150 & +0.0113 & 0.03827 & 0.00028 & 0.00024 \\
 & MSE & +0.0110 & +0.0117 & +0.0152 & +0.0089 & 0.03824 & 0.00029 & 0.00024 \\
 & \cellcolor{green!15} \textbf{VCA} & \cellcolor{green!15} +0.0046 & \cellcolor{green!15} +0.0053 & \cellcolor{green!15} +0.0083 & \cellcolor{green!15} +0.0070 & \cellcolor{green!15} 0.03814 & \cellcolor{green!15} 0.00029 & \cellcolor{green!15} 0.00025 \\
\bottomrule
\end{tabular}
}
\end{table}

\begin{table}[!htb]
\centering
\small
\caption{Raw per-cell results at $H{=}32$ (Chronos). $\sigma^2$-MAE columns are reported to 5 decimal places, 1 for the Chronos-VCA rows; all other columns to 4.}
\label{tab:chronos_raw_h32}
\resizebox{\textwidth}{!}{%
\begin{tabular}{llccccccc}
\toprule
DS & Method & price RankIC & price IC & ret.\ RankIC & ret.\ IC & ACF$^2$ gap & $\sigma^2$-MAE\textsubscript{\textsc{fore}} & $\sigma^2$-MAE\textsubscript{\textsc{vol}} \\
\midrule
\multirow{4}{*}{\rotatebox[origin=c]{90}{CSI300}} & Frozen & +0.0042 & +0.0040 & +0.0223 & +0.0200 & 0.01808 & 0.01630 & 0.01410 \\
 & CE & +0.0195 & +0.0192 & +0.0873 & +0.0933 & 0.01779 & 0.01660 & 0.01500 \\
 & MSE & +0.0118 & +0.0124 & +0.0721 & +0.0493 & 0.01763 & 0.01670 & 0.01530 \\
 & \cellcolor{green!15} \textbf{VCA} & \cellcolor{green!15} -0.0173 & \cellcolor{green!15} -0.0197 & \cellcolor{green!15} +0.0058 & \cellcolor{green!15} -0.0035 & \cellcolor{green!15} 0.01719 & \cellcolor{green!15} 21.7 & \cellcolor{green!15} 51.7 \\
\midrule
\multirow{4}{*}{\rotatebox[origin=c]{90}{CSI500}} & Frozen & +0.0534 & +0.0458 & +0.0990 & +0.0622 & 0.01798 & 0.02180 & 0.01900 \\
 & CE & -0.0142 & -0.0127 & +0.0856 & +0.0922 & 0.01768 & 0.02190 & 0.01970 \\
 & MSE & -0.0757 & -0.0755 & +0.0758 & +0.0420 & 0.01759 & 0.02220 & 0.02060 \\
 & \cellcolor{green!15} \textbf{VCA} & \cellcolor{green!15} +0.1274 & \cellcolor{green!15} +0.1295 & \cellcolor{green!15} +0.0449 & \cellcolor{green!15} +0.0474 & \cellcolor{green!15} 0.01766 & \cellcolor{green!15} 0.02190 & \cellcolor{green!15} 0.01930 \\
\midrule
\multirow{4}{*}{\rotatebox[origin=c]{90}{Crypto}} & Frozen & +0.0049 & +0.0051 & -0.0160 & -0.0212 & 0.01764 & 0.00060 & 0.00046 \\
 & CE & +0.0162 & +0.0147 & -0.0131 & -0.0162 & 0.01745 & 0.00064 & 0.00046 \\
 & MSE & -0.0036 & -0.0048 & -0.0024 & -0.0045 & 0.01752 & 0.00063 & 0.00047 \\
 & \cellcolor{green!15} \textbf{VCA} & \cellcolor{green!15} +0.0011 & \cellcolor{green!15} -0.0012 & \cellcolor{green!15} +0.0027 & \cellcolor{green!15} -0.0021 & \cellcolor{green!15} 0.01728 & \cellcolor{green!15} 15.0 & \cellcolor{green!15} 93.4 \\
\bottomrule
\end{tabular}
}
\end{table}

\end{document}